\documentclass[twocolumn]{article}

\usepackage{color}
\usepackage{bm}
\usepackage[pdftex]{graphicx}
\usepackage{amsthm}
\usepackage{amsmath}
\usepackage{amssymb}
\usepackage{url}
\usepackage{natbib}
\usepackage[normalem]{ulem}
\usepackage{captdef}

\newtheorem{theorem}{Theorem}[section]

\newtheorem{lemma}[theorem]{Lemma}
\newtheorem{corollary}[theorem]{Corollary}
\theoremstyle{plain}
\newtheorem{definition}[theorem]{Definition}

\newtheorem{remark}[theorem]{Remark}

\newcommand{\argmax}{\mathop{\rm argmax}\limits}
\newcommand{\argmin}{\mathop{\rm argmin}\limits}

\newcommand{\targmin}{{\rm argmin}}

\newcommand{\sqtimes}{{\raisebox{0.08em}{$\times$}\hspace{-0.79em}\Box}}

\newcommand{\red}[1]{\textcolor{red}{#1}}
\newcommand{\blue}[1]{\textcolor{blue}{#1}}

\title{Distributionally Robust Safe Sample Elimination under Covariate Shift}

\author{
	Hiroyuki Hanada${}^*$ \and
	Tatsuya Aoyama${}^{\dagger}$ \and
	Satoshi Akahane${}^{\dagger}$ \and
	Tomonari Tanaka${}^{\dagger}$ \and
	Yoshito Okura${}^{\dagger}$ \and
	Yu Inatsu${}^{\ddagger}$ \and
	Noriaki Hashimoto${}^*$ \and
	Shion Takeno${}^{{\dagger}*}$ \and
	Taro Murayama${}^{\#}$ \and
	Hanju Lee${}^{\#}$ \and
	Shinya Kojima${}^{\#}$ \and
	Ichiro Takeuchi${}^{{\dagger}*}$
	\\
	\small{* RIKEN, $\dagger$ Nagoya University, $\ddagger$ Nagoya Institute of Technology, \# DENSO CORPORATION}
}

\begin{document}
\maketitle

\begin{abstract}
We consider a machine learning setup where one training dataset is used to train multiple models across slightly different data distributions.
This occurs when customized models are needed for various deployment environments.
To reduce storage and training costs, we propose the DRSSS method, which combines distributionally robust (DR) optimization and safe sample screening (SSS).
The key benefit of this method is that models trained on the reduced dataset will perform the same as those trained on the full dataset for all possible different environments.
In this paper, we focus on covariate shift as a type of data distribution change and demonstrate the effectiveness of our method through experiments.


\end{abstract}

\section{INTRODUCTION} \label{sec:introduction}

In this paper, we address the problem of reducing dataset size for training many predictive models across slightly different data distributions.
This is common in applications requiring customized models for varied environments, such as healthcare prediction for different patient groups, financial risk assessment across market conditions, or product refinement for specific regions.
In these cases, efficiently training a customized model for each environment is desirable, and keeping the dataset small can help reduce both storage and training costs.

Our focus is on eliminating a subset of training instances to create a smaller dataset, ensuring that the model trained on this reduced set performs as well as one trained on the full dataset when environmental changes remain within a certain range.
To address this, we propose \emph{distributionally robust safe sample screening (DRSSS)}, which integrates distributionally robust (DR) learning~\citep{goh2010distributionally,delage2010distributionally,chen2021distributionally} and safe sample screening (SSS)~\citep{ogawa2013safe,shibagaki2016simultaneous}.
DR learning builds models that are robust to changes in input distribution, while SSS helps efficiently train sample-sparse models such as support vector machine (SVM) by identifying and removing instances that do not affect the solution (called non-support vectors (SVs) in SVM).

As a type of environmental change, we consider covariate shift, where the input distribution varies between training and testing.
If the test distribution is known, the covariate shift problem can be addressed by performing instance weighted learning, using the density ratio between the test and training distributions as the weight.
For sample-sparse models such as SVM, it is rather straightforward to extend existing SSS methods for instance weighted sample sparse model if the weights are known. 
However, when the test distribution is unknown, we need to consider weighted training where the weights vary within a certain range, making it challenging to identify instances that do not affect all possible customized models trained with a variety of instance weights.

Namely, in this problem setting, our technical challenge is formulated as identifying training instances that do not affect all possible customized models in a weighted sample-sparse model, where the weight distribution varies within a certain range.
If we can identify such a subset of training instances, we can remove them in advance without altering all possible customized models in different environments, thereby reducing both storage and training costs.
To confirm the effectiveness of the proposed DRSSS method, which is designed to address this challenge, we conduct experiments not only with simple kernels, such as the Gaussian kernel, but also with the neural tangent kernel (NTK) \citep{jacot2018neural}, which achieves performance comparable to state-of-the-art deep learning models.
Figure~\ref{fig:safe-sample-screening-DR} illustrates the problem setting and the proposed DRSSS method.

\begin{figure}[t]
\begin{center}
\includegraphics[width=\hsize]{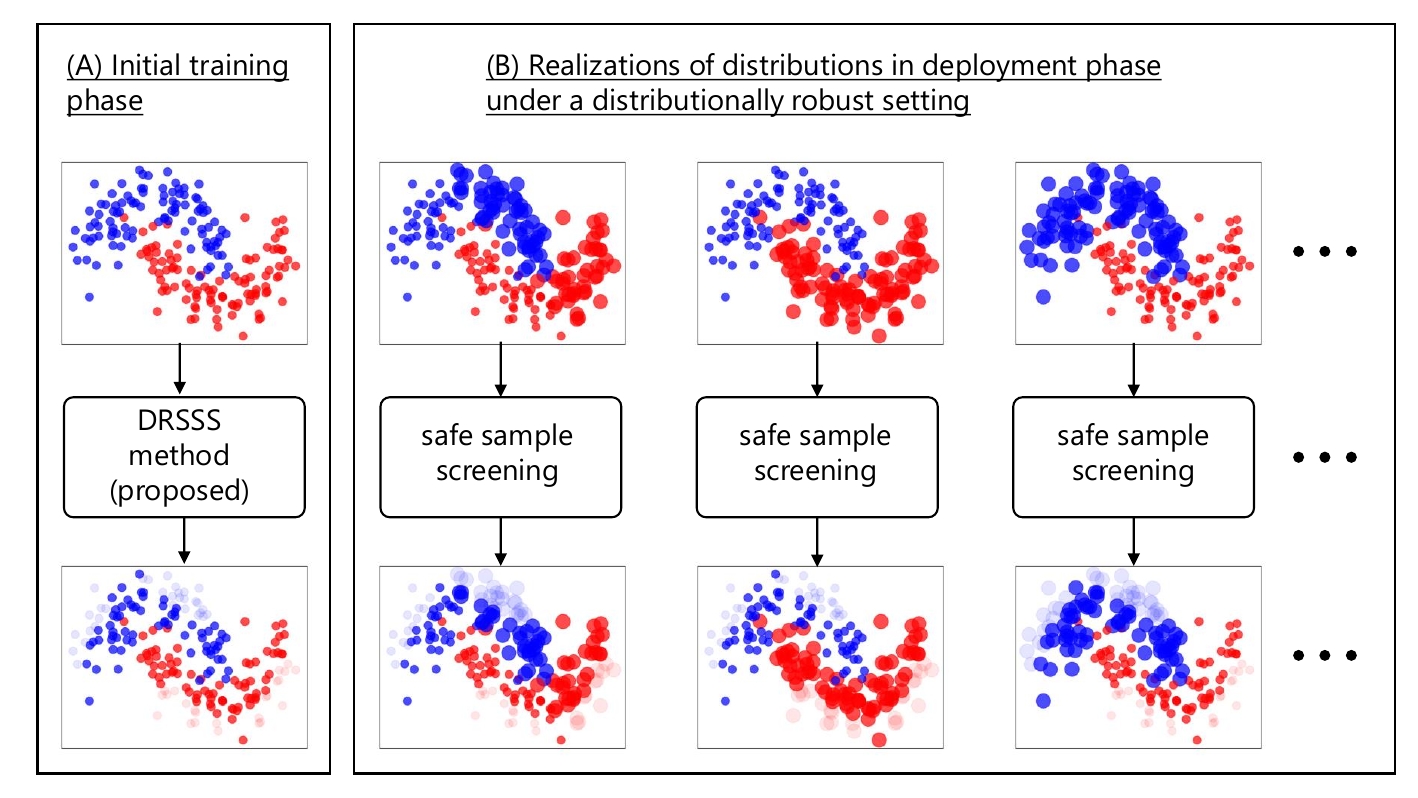}
\caption{
Schematic illustration of the proposed DRSSS method in a toy binary classification problem with a sample-sparse classifier such as SVM.
(A) Suppose that we have a dataset of binary classification (classes described by colors),
and we compute the proposed DRSSS in the initial training phase.
Then, it can identify a set of samples that do not influence the optimal solutions in the deployment phase, regardless of any changes within a specified range in the input distribution (points with thin colors).
(B) This means that, for any input distribution of the deployment phase within a specified range
(depicted by the differing sizes of points),
the set of samples screened out by the DRSSS method is always identified as no influence
(screened samples in a specific deployment is always a superset of that in DRSSS).
}
\label{fig:safe-sample-screening-DR}
\end{center}
\end{figure}

\paragraph{Related Works} \label{sec:related}

The DR setting has been explored in various machine learning problems, aiming to enhance model robustness against data distribution variations.
A DR learning problem is typically formulated as a worst-case optimization problem since the goal is to ensure model performance under the worst-case data distribution within a specified range.
Consequently, a variety of optimization techniques tailored to DR learning have been investigated within both the machine learning and optimization communities~\citep{goh2010distributionally,delage2010distributionally,chen2021distributionally}.
The proposed DRSSS method is inspired by such DR learning methods, however, as stated previously,
DRSSS does not consider DR of the learning computation but of the SSS computations.

DRSSS mainly aims to reduce samples for computational resource problems, however,
sample reductions are of practical significance also in other several contexts.
For example, in the context of continual learning (see, e.g., \citet{wang2022memory}), it is crucial to effectively manage data by selectively retaining and discarding samples, especially in anticipation of changes in future data distributions.
Incorrect deletion of essential data can lead to \emph{catastrophic forgetting}~\citep{kirkpatrick2017overcoming}, a phenomenon where a model, after being trained on new data, quickly loses information previously learned from older datasets.
%

We employ SSS as a criterion of removing samples, however,
methods called the core set selection \citep{sener2018active,mirzasoleiman2020coresets}
and the data distillation \citep{geng2023dataset,lei2024comprehensive} are also studied.
In the context of data distillation, when we newly created a smaller dataset than the original dataset,
there exists several criteria to measure the similarily of these two datasets:
the model parameter, the gradient of the objective function of the learning, and
the data distribution \citep{lei2024comprehensive}.
SSS belongs to the first criterion, however, by appropriately limiting the class of the problem
we can remove samples {\em without changing the model parameter at all}.

Safe screening (SS) refers to optimization techniques in sparse learning that identify and exclude unnecessary samples or features from the learning process, thereby reducing computational cost without altering the final trained model.
Initially, safe \emph{feature} screening was introduced by \citet{ghaoui2012safe} for the Lasso, followed by safe \emph{sample} screening by \citet{ogawa2013safe} for the SVM.
Among various SS methods developed so far, the most commonly used is based on the duality gap~\citep{fercoq2015mind,ndiaye2015gap,shibagaki2016simultaneous}, which is also used for the proposed DRSSS method.
Over the past decade, SS has seen diverse developments, including methodological improvements and expanded application scopes~\citep{shibagaki2016simultaneous,okumura2015quick,nakagawa2016safe,ren2018safe,zhao2019improved,zhai2020safe,wang2022safe,yoshida2023efficient}.
Unlike other SSS studies that primarily focused on reducing computational costs, this study adopts SSS for a different purpose.
We employ SSS to identify unnecessary samples from an infinitely large number of learning problems, or more specifically, across infinitely large numbers of scenarios where data distribution varies within a specified range.
To our knowledge, no existing studies have utilized SSS within the DR framework.

\paragraph{Contributions}
Our contributions are summarized as follows.
First, by effectively combining DR and SSS, we propose a novel method called the DRSSS method that can identify samples that do not affect the optimal solutions in the deployment phase, regardless of how the distribution changes within a specified range.
Although the DRSSS method is developed for convex empirical risk minimization problems, we have extended the proposed method through kernelization to be applicable to complex nonlinear models.
Using NTK, this extension enables the application of DRSSS to practical deep learning models.
Finally, through numerical experiments, we verify the effectiveness of the proposed DRSSS method.

\paragraph{Notations}
Notations used in this paper are described in Table~\ref{tab:definitions}.
\begin{table*}[t]
\caption{Notations used in the paper. $\mathbb{R}$: all real numbers, $\mathbb{N}$: all positive integers, $n, n^\prime, p\in\mathbb{N}$, $f: \mathbb{R}^n\to\mathbb{R}\cup\{+\infty\}$: convex function, $M\in\mathbb{R}^{n\times n^\prime}$: matrix, $\bm v\in\mathbb{R}^n$: vector.}
\label{tab:definitions}
\centering
{\small
\begin{tabular}{ll}
\hline
$m_{ij} \in \mathbb{R}$ (small case of matrix variable)
	& the element at the $i^\mathrm{th}$ row and the $j^\mathrm{th}$ column of $M$ \\
$v_i \in \mathbb{R}$ (nonbold font of vector variable)
	& the $i^\mathrm{th}$ element of $\bm v$\\
$M_{i:}\in\mathbb{R}^{1\times n^\prime}$ & the $i^\mathrm{th}$ row of $M$ \\
$M_{:j}\in\mathbb{R}^{n\times 1}$ & the $j^\mathrm{th}$ column of $M$ \\
$[n]$ & $\{1, 2, \dots, n\}$ \\
$\mathbb{R}_{\geq 0}$ & all nonnegative real numbers \\
$\otimes$ & elementwise product \\
\emph{diagonal matrix}: $\mathrm{diag}(\bm v) \in\mathbb{R}^{n\times n}$
	& $(\mathrm{diag}(\bm v))_{ii} = v_i$ and $(\mathrm{diag}(\bm v))_{ij} = 0$ ($i\neq j$) \\
$\bm v\sqtimes M \in\mathbb{R}^{n\times n^\prime}$ & $\mathrm{diag}(\bm v) M$ \\
$\bm 0_n \in \mathbb{R}^n$ & $[0, 0, \dots, 0]^\top$ (vector of size $n$) \\
$\bm 1_n \in \mathbb{R}^n$ & $[1, 1, \dots, 1]^\top$ (vector of size $n$) \\
\emph{$p$-norm}: $\|\bm v\|_p \in \mathbb{R}_{\geq 0}$ & $(\sum_{i=1}^n v_i^p)^{1/p}$ \\
\emph{subgradient}: $\partial f(\bm v) \subseteq \mathbb{R}^n$
	& $\{ \bm g\in\mathbb{R}^n \mid \forall\bm v^\prime\in\mathbb{R}^n:~f(\bm v^\prime) - f(\bm v)\geq \bm g^\top(\bm v^\prime - \bm v) \}$ \\
${\cal Z}[f] \subseteq\mathbb{R}^n$ & $\{ \bm v^\prime\in\mathbb{R}^n \mid \partial f(\bm v^\prime) = \{\bm 0_n\} \}$ \\
\emph{convex conjugate}: $f^*(\bm v) \in \mathbb{R}\cup\{+\infty\}$
	& $\sup_{\bm v^\prime\in\mathbb{R}^n} (\bm v^\top \bm v^\prime - f(\bm v^\prime))$ \\
``$f$ is $\kappa$-strongly convex'' ($\kappa>0$)
	& $f(\bm v) - (\kappa/2)\|\bm v\|_2^2$ is convex with respect to $\bm v$ \\
``$f$ is $\mu$-smooth'' ($\mu > 0$)
	& $\|f(\bm v) - f(\bm v^\prime)\|_2\leq \mu\|\bm v - \bm v^\prime\|_2$ for any $\bm v, \bm v^\prime\in\mathbb{R}^n$ \\
\hline
\end{tabular}
}
\end{table*}

\section{PRELIMINARIES} \label{sec:preliminaries}

\subsection{Weighted Regularized Empirical Risk Minimization} \label{sec:weighted-RERM}

We consider machine learning models for linear prediction based on the
\emph{weighted regularized empirical risk minimization} (weighted RERM).
Under the covariate shift setting, it is known that we should set the weight of the sample
as the ratio of the probability of the sample appearing in the distribution after change
to that before change \citep{shimodaira2000improving,sugiyama2007covariate}.
So we employ the formulation.
%
We note that this ``linear prediction'' includes the \emph{kernel method}, which is discussed in
Section \ref{sec:safe-screening-kernelized}.


\begin{definition}\label{def:WRERM}
Given $n$ training samples of $d$-dimensional input variables, scalar output variables
and scalar sample weights, denoted by $X\in\mathbb{R}^{n\times d}$, $\bm y\in\mathbb{R}^n$ and
$\bm w\in\mathbb{R}_{\geq 0}^n$, respectively,
the training computation by weighted RERM for linear prediction is formulated as follows:
\begin{align}
& \bm\beta^{*(\bm w)} := \argmin_{\bm\beta\in\mathbb{R}^d} P_{\bm w}(\bm\beta),
	\quad\text{where} \\
& P_{\bm w}(\bm\beta) := \sum_{i=1}^n w_i \ell_{y_i}(\check{X}_{i:}\bm\beta) + \rho(\bm\beta). \label{eq:primal}
\end{align}
Here, $\ell_y: \mathbb{R}\to\mathbb{R}$ is a convex {\em loss function}\footnote{For $\ell_y(t)$, we assume that only $t$ is a variable of the function ($y$ is assumed to be a constant) when we take its subgradient or convex conjugate.}, $\rho: \mathbb{R}^d\to\mathbb{R}$ is a convex {\em regularization function}, and $\check{X}\in\mathbb{R}^{n\times d}$ is a matrix calculated from $X$ and $\bm y$ and determined depending on $\ell_y$.
In this paper, unless otherwise noted, we consider binary classifications ($\bm y\in\{-1, +1\}^n$)
with $\check{X} := \bm y \sqtimes X$.
For regressions ($\bm y\in\mathbb{R}^n$) we usually set $\check{X} := X$.
\end{definition}
Here, $\bm\beta^{*(\bm w)}$ is the linear prediction coefficients: given a sample $\hat{\bm x}\in\mathbb{R}^d$, its prediction is assumed to be done by $\hat{\bm x}^\top\bm\beta^{*(\bm w)}$.
Since $\ell_y$ and $\rho$ are convex, we can easily confirm that $P_{\bm w}(\bm\beta)$ is convex with respect to $\bm\beta$.
\begin{remark} \label{rem:intercept}
In machine learning models, we often introduce the \emph{intercept} (the common coefficient for any sample)
in the model parameter.
In this formulation we assume that $X_{:d} = \bm 1_n$ to regard $\beta^{*(\bm w)}_d$ (the last element) as the intercept.
\end{remark}

Applying {\em Fenchel's duality theorem} (Appendix \ref{app:fenchel}), we have the following {\em dual problem} of \eqref{eq:primal}:
\begin{align}
& \bm\alpha^{*(\bm w)} := \argmax_{\bm\alpha\in\mathbb{R}^n} D_{\bm w}(\bm\alpha),
	\quad\text{where} \\
& D_{\bm w}(\bm\alpha) := -\sum_{i=1}^n w_i \ell^*_{y_i}(-\alpha_i) - \rho^*((\bm w\sqtimes\check{X})^\top \bm\alpha). \label{eq:dual}
\end{align}
The relationship between the original problem \eqref{eq:primal} (called the \emph{primal} problem) and the dual problem \eqref{eq:dual} are described as follows:
\begin{align}
& P_{\bm w}(\bm\beta^{*(\bm w)}) = D_{\bm w}(\bm\alpha^{*(\bm w)}), \label{eq:strong-duality}\\
& \bm\beta^{*(\bm w)} \in \partial\rho^*((\bm w\sqtimes\check{X})^\top \bm\alpha^{*(\bm w)}), \label{eq:KKT-dual2primal}\\
& \forall i\in[n]:\quad -\alpha^{*(\bm w)}_i \in \partial\ell_{y_i}(\check{X}_{i:}\bm\beta^{*(\bm w)}). \label{eq:KKT-primal2dual}
\end{align}

Note that, depending on the choice of the loss function and the regularization function,
we need to avoid using $\bm\alpha$ such that $D(\bm\alpha) = +\infty$.
See Appendix \ref{app:loss-functions} for details.

\subsection{Sample-Sparse Loss Functions} \label{sec:sample-sparsity}

In weighted RERM, we call that a loss function $\ell_y$ induces \emph{sample-sparsity} if elements in $\bm\alpha^{*(\bm w)}$ are easy to become zero.
As implied from \eqref{eq:KKT-dual2primal}, if $\alpha^{*(\bm w)}_i = 0$,
the $i^\mathrm{th}$ sample (i.e., the $i^\mathrm{th}$ row of $\check{X}$) is unnecessary
to represent the primal model parameter $\bm\beta^{*(\bm w)}$ and therefore we can remove them.
Due to \eqref{eq:KKT-primal2dual}, the sample-sparsity can be achieved by $\ell_y$ such that
${\cal Z}[\ell_y] := \{ t\in\mathbb{R} \mid \partial \ell_y(t) = \{0\} \}$ is not a point but an interval.

In case of binary classification problems ($y\in\{-1, +1\}$),
the \emph{hinge loss} $\ell_{y}(t) = \max\{0, 1-t\}$
is a famous sample-sparse loss function.
In fact, ${\cal Z}[\ell_y] = (1, +\infty)$ for the hinge loss.
Other examples are presented in Table \ref{tab:loss-functions} in Appendix \ref{app:loss-functions}.

\section{DISTRIBUTIONALLY ROBUST SAFE SAMPLE SCREENING} \label{sec:safe-screening}

In this section we show the proposed method DRSSS for weighted RERM with two steps.
First, in Sections \ref{sec:safe-sample-screening} we show SSS rules (rule to remove samples)
for weighted RERM but not DR setup.
Here, ``DR setup'' means that the weight is assumed to be changed in a predetermined range.
To do this, we extended existing SSS rules in \citep{ndiaye2015gap,shibagaki2016simultaneous}.
Then we derive DRSSS rules in Section \ref{sec:DRSS}.

\subsection{Existing (Non-DR) Safe Sample Screening} \label{sec:safe-sample-screening}

We consider identifying training samples that do not affect the learned primal model parameter $\bm\beta^{*(\bm w)}$.
As discussed in Section \ref{sec:sample-sparsity},
such samples can be identified as $i\in[n]$ such that $\alpha^{*(\bm w)}_i = 0$.
However, since computing $\bm\alpha^{*(\bm w)}$ is as costly as $\bm\beta^{*(\bm w)}$, we would like to judge whether $\alpha^{*(\bm w)}_i = 0$ or not without computing exact $\bm\alpha^{*(\bm w)}$.
To solve the problem, the SSS first considers identifying the possible range ${\cal B}^{*(\bm w)}\subset\mathbb{R}^d$ such that $\bm\beta^{*(\bm w)}\in{\cal B}^{*(\bm w)}$ is assured.
Then, with ${\cal B}^{*(\bm w)}$ and \eqref{eq:KKT-primal2dual}, we can conclude that the $i^\mathrm{th}$ training sample do not affect the training result $\bm\beta^{*(\bm w)}$ if
$\bigcup_{\bm\beta\in{\cal B}^{*(\bm w)}} \partial\ell_{y_i}(\check{X}_{i:}\bm\beta) = \{0\}$.

First we show how to compute ${\cal B}^{*(\bm w)}$.
In this paper we adopt the computation methods that is available when the regularization function $\rho$ in $P_{\bm w}$ (and also $P_{\bm w}$ itself) of \eqref{eq:primal} are strongly convex.

\begin{lemma}\label{lem:gap-sphere-primal}
Suppose that $\rho$ in $P_{\bm w}$ (and also $P_{\bm w}$ itself) of \eqref{eq:primal} are $\kappa$-strongly convex.
Then, for any $\hat{\bm\beta}\in\mathbb{R}^d$ and $\hat{\bm\alpha}\in\mathbb{R}^n$, we can assure that the following ${\cal B}^{*(\bm w)}\subset\mathbb{R}^d$ must satisfy $\bm\beta^{*(\bm w)}\in{\cal B}^{*(\bm w)}$:
\begin{align*}
& {\cal B}^{*(\bm w)} := \{ \bm\beta\in\mathbb{R}^d \mid \| \bm\beta - \hat{\bm\beta} \|_2 \leq r(\bm w, \kappa, \hat{\bm\beta}, \hat{\bm\alpha}) \}, \\
& \text{where}\quad
	r(\bm w, \kappa, \hat{\bm\beta}, \hat{\bm\alpha}) := \sqrt{(2/\kappa)[P_{\bm w}(\hat{\bm\beta}) - D_{\bm w}(\hat{\bm\alpha})]}.
\end{align*}
\end{lemma}
The proof is done by extending \citep{ndiaye2015gap,shibagaki2016simultaneous} to our setup; see Appendix \ref{app:gap-sphere-primal}.
The quantity $P_{\bm w}(\hat{\bm\beta}) - D_{\bm w}(\hat{\bm\alpha})$ is known as the \emph{duality gap}, which must be nonnegative due to \eqref{eq:strong-duality}.
Then we obtain the following SSS rule, which is referred to as the \emph{gap safe screening rule} \citep{ndiaye2015gap}:
\begin{lemma}\label{lem:gap-sample-screening}
Under the same assumptions as Lemma \ref{lem:gap-sphere-primal}, $\alpha_i^{*(\bm w)} = 0$ is assured
(i.e., the $i^\mathrm{th}$ training sample does not affect the training result $\bm\beta^{*(\bm w)}$)
if there exists $\hat{\bm\beta}\in\mathbb{R}^d$ and $\hat{\bm\alpha}\in\mathbb{R}^n$ such that
\begin{align*}
& [\check{X}_{i:}\hat{\bm\beta} - \|\check{X}_{i:}\|_2 r(\bm w, \kappa, \hat{\bm\beta}, \hat{\bm\alpha}), \\
	&\quad
	\check{X}_{i:}\hat{\bm\beta} + \|\check{X}_{i:}\|_2 r(\bm w, \kappa, \hat{\bm\beta}, \hat{\bm\alpha})] \subseteq {\cal Z}[\ell_{y_i}].
\end{align*}
\end{lemma}
The proof is presented in Appendix \ref{app:gap-sample-screening}.

\begin{remark}
When we employ the intercept in the model (Remark \ref{rem:intercept}),
we often define that the regularization function $\rho(\bm\beta)$ do not regularize $\beta_d$
(the model parameter for the intercept), that is,
$\rho(\bm\beta)$ is constant with respect to $\beta_d$.
However, to apply Lemmas \ref{lem:gap-sphere-primal} and \ref{lem:gap-sample-screening},
we need to define that $\rho(\bm\beta)$ does regularize $\beta_d$,
because $\rho$ itself is required to be $\kappa$-strongly convex.
\end{remark}

\subsection{SSS under covariate-shift DR setup} \label{sec:DRSS}

Here we consider the DR version of the SSS rule above, that is, we identify samples that meets
the SSS rule for any distribution we consider.
From the discussion in Section \ref{sec:weighted-RERM}, in covariate-shift setup,
``any distribution we consider'' can be redefined as the range of sample weights.

\begin{definition}[DRSSS under covariate shift] \label{def:safe-screening-robust-change}
Given $X\in\mathbb{R}^{n\times d}$, $\bm y\in\mathbb{R}^n$ and ${\cal W}\subset\mathbb{R}_{\geq 0}^n$,
the \emph{DRSSS for ${\cal W}$ in weighted RERM} is defined as finding unnecessary samples for any $\bm w\in{\cal W}$, or more specifically, finding $i\in[n]$ that satisfies Lemma \ref{lem:gap-sample-screening} for any $\bm w\in{\cal W}$.
\end{definition}

\begin{theorem} \label{thm:safe-screening-robust}
If $\rho$ in $P_{\bm w}$ \eqref{eq:primal} is $\kappa$-strongly convex,
DRSSS in Definition \ref{def:safe-screening-robust-change} is achieved as follows:
\begin{enumerate}
\item Choose $\tilde{\bm w}\in{\cal W}$.
\item Compute \eqref{eq:primal}\&\eqref{eq:KKT-primal2dual} or \eqref{eq:dual}\&\eqref{eq:KKT-dual2primal} for $\tilde{\bm w}$ to learn $\bm\beta^{*(\tilde{\bm w})}$ and $\bm\alpha^{*(\tilde{\bm w})}$. 
\item Identify $i\in[n]$ that satisfies Lemma \ref{lem:gap-sample-screening} with
	$\hat{\bm\beta}\gets\bm\beta^{*(\tilde{\bm w})}$ and $\hat{\bm\alpha}\gets\bm\alpha^{*(\tilde{\bm w})}$
	and for any $\bm w\in{\cal W}$. Such $i$ assures $\alpha_i^{*(\bm w)} = 0$~$\forall\bm w\in{\cal W}$.
\end{enumerate}
Here, the last procedure can be computed as follows:
\begin{align}
& [\check{X}_{i:}\bm\beta^{*(\tilde{\bm w})} - \|\check{X}_{i:}\|_2 R,~
	\check{X}_{i:}\bm\beta^{*(\tilde{\bm w})} + \|\check{X}_{i:}\|_2 R] \subseteq {\cal Z}[\ell_{y_i}],
	\label{eq:safe-sample-screening-rule} \\
& \text{where}\quad
	R := \max_{\bm w\in{\cal W}} r(\bm w, \kappa, \bm\beta^{*(\tilde{\bm w})}, \bm\alpha^{*(\tilde{\bm w})}) \nonumber\\
	&\quad
	= \sqrt{(2/\kappa) \max_{\bm w\in{\cal W}} [P_{\bm w}(\bm\beta^{*(\tilde{\bm w})}) - D_{\bm w}(\bm\alpha^{*(\tilde{\bm w})})]}. \label{eq:safe-sample-screening-radius}
\end{align}
\end{theorem}

In Theorem \ref{thm:safe-screening-robust},
the key point is the maximization in \eqref{eq:safe-sample-screening-radius},
which is a maximization of convex function and therefore is difficult in general.
In fact,
\begin{align}
& P_{\bm w}(\bm\beta^{*(\tilde{\bm w})}) - D_{\bm w}(\bm\alpha^{*(\tilde{\bm w})}) \nonumber\\
& = \underbrace{\sum_{i=1}^n w_i [\ell_{y_i}(\check{X}_{i:}\bm\beta^{*(\tilde{\bm w})}) + \ell^*_{y_i}(-\alpha^{*(\tilde{\bm w})}_i)]}_{\text{linear w.r.t.}~\bm w}
		+ \rho(\bm\beta^{*(\tilde{\bm w})}) \nonumber\\
		&\quad + \underbrace{\rho^*((\bm w\sqtimes\check{X})^\top \bm\alpha^{*(\tilde{\bm w})})}_{\text{convex w.r.t.}~\bm w},
		\label{eq:duality-gap}
\end{align}
is convex with respect to $\bm w$. (Note that $\rho^*$ is assured to be convex.)
However, depending on the choice of the regularization function $\rho$
and the range of weight changes ${\cal W}$, it may be explicitly computed.
So, in Section \ref{sec:DRSS-examples} we show specific calculations for some typical setups.

\section{SPECIFIC COMPUTATIONS OF DRSSS FOR L2-REGULARIZATION} \label{sec:DRSS-examples}

In this section we discuss specific computations of DRSSS in Theorem \ref{thm:safe-screening-robust},
which depends on the choice of the loss function $\ell_y$, the regularization function $\rho$,
and the distribution changes (i.e., range of weight changes ${\cal W}$).
Especially, we present how to compute the quantities to conduct DRSSS:
computation of \eqref{eq:safe-sample-screening-rule} and maximization of \eqref{eq:duality-gap} in the theorem.

First we can see that the choice of the loss function does not affect the required computations
for DRSSS so much, since the terms including the loss function $\ell_y$ and its convex conjugate $\ell^*_y$
in \eqref{eq:safe-sample-screening-rule} and \eqref{eq:duality-gap} are calculated only by known quantities
$\check{X}$, $\bm\beta^{*(\tilde{\bm w})}$ and $\bm\alpha^{*(\tilde{\bm w})}$.
So, after determining $\ell_y$, we have only to apply functions described in Table \ref{tab:loss-functions}.

On the other hand, the choice of the regularization function and the range of weight changes largely affect
the difficulty of the computations.
So we mainly discuss the use of L2-regularization and show the algorithm to compute them.
Note that, it is not obvious to extend DRSSS for L2-regularization to other regularization functions;
discussions are presented in Appendix \ref{app:regularization-functions}.

\subsection{DRSSS for L2-Regularization and Weight Changes in L2-Norm} \label{sec:safe-screening-l2-regularization}

As stated in Lemma \ref{lem:gap-sphere-primal},
DRSSS requires the regularization function to be $\kappa$-strongly convex.
So the \emph{L2-regularization}
$\rho(\bm\beta) := \frac{\lambda}{2}\|\bm\beta\|_2^2$ ($\lambda > 0$) can be applicable for DRSSS.
In addition, we assume that the range of weight changes ${\cal W}$
is defined by L2-norm ${\cal W} := \{ \bm w \mid \| \bm w - \tilde{\bm w} \|_2\leq S \}$ ($S > 0$).

With L2-regularization, since $\rho^*(\bm v) = \frac{1}{2\lambda} \|\bm v\|_2^2$,
the duality gap \eqref{eq:duality-gap} is calculated as
\begin{align}
& \sum_{i=1}^n w_i [\ell_{y_i}(\check{X}_{i:}\bm\beta^{*(\tilde{\bm w})}) + \ell^*_{y_i}(-\alpha^{*(\tilde{\bm w})}_i)] + \frac{\lambda}{2}\|\bm\beta^{*(\tilde{\bm w})}\|_2^2 \nonumber\\
	&\quad
	+ \frac{1}{2\lambda} \bm w^\top (\bm\alpha^{*(\tilde{\bm w})}\sqtimes\check{X}) (\bm\alpha^{*(\tilde{\bm w})}\sqtimes\check{X})^\top \bm w. \label{eq:duality-gap-l2reg}
\end{align}

Since $(\bm\alpha^{*(\tilde{\bm w})}\sqtimes\check{X}) (\bm\alpha^{*(\tilde{\bm w})}\sqtimes\check{X})^\top$
is positive semidefinite, we can see that \eqref{eq:duality-gap-l2reg} is a quadratic convex function with respect to $\bm w$.
As a method to solve a constrained maximization of \eqref{eq:duality-gap-l2reg}
in ${\cal W} := \{ \bm w \mid \| \bm w - \tilde{\bm w} \|_2\leq S \}$,
let us apply the \emph{method of Lagrange multipliers}.
Then we can maximize it since all stationary points can be algorithmically computed. In fact,
\begin{theorem} \label{thm:maximize-convex-quadratic-sketch}
The following maximization can be solved algorithmically:
\begin{align}
& \max_{\bm w\in{\cal W}} \bm w^\top A \bm w + 2\bm b^\top\bm w,
	\quad\text{where}  \label{eq:maximize-convex-quadratic} \\
& {\cal W} := \{ \bm w\in\mathbb{R}^n \mid \|\bm w - \tilde{\bm w}\|_2\leq S \}, \nonumber\\
& \tilde{\bm w}\in\mathbb{R}^n,
	\quad \bm b\in\mathbb{R}^n,
	\quad S > 0, \nonumber\\
& A\in\mathbb{R}^{n\times n}:~\text{symmetric, positive semidefinite, nonzero}. \nonumber
\end{align}
\end{theorem}
See Appendix \ref{app:maximize-convex-quadratic} for the proof.
After maximizing \eqref{eq:duality-gap-l2reg} by Theorem \ref{thm:maximize-convex-quadratic-sketch},
we can compute \eqref{eq:safe-sample-screening-radius} and therefore the DRSSS rule in
Theorem \ref{thm:safe-screening-robust} can be computed.

Although the maximization of Theorem \ref{thm:maximize-convex-quadratic-sketch} requires $O(n^3)$ time,
we can think that the cost can be recovered by the future training computations.

\subsection{DRSSS for L2-Regularization Using Kernels} \label{sec:safe-screening-kernelized}

In this section we show the condition to apply both the kernel method and DRSSS,
and show the explicit computations when L2-regularization is employed.

For detailed formulation of kernels, see Chapter 2 of \citep{scholkopf2004kernel} for example.
%
In this paper, for (weighted) RERM, we define a computation is \emph{kernelized} if input variables for any instance (e.g., elements of $X$) or primal model parameters (e.g., $\bm\beta^{*(\tilde{\bm w})}$) cannot be explicitly obtained, yet inner products between samples (e.g., $X_{i:} X_{i^\prime :}^\top$) are accessible.
In Appendix \ref{app:regularization-functions} we will discuss the condition of regularization functions
so that weighted RERM can accept both kernelization and SSS (although it is not obvious to extend to DRSSS).

To compute DRSSS with kernelized setup,
let us see \eqref{eq:safe-sample-screening-rule} and \eqref{eq:duality-gap-l2reg}.
Noticing that the relationship \eqref{eq:KKT-dual2primal} for L2-regularization is calculated as
$\bm\beta^{*(\tilde{\bm w})}
	= \frac{1}{\lambda}(\tilde{\bm w}\sqtimes\check{X})^\top \bm\alpha^{*(\tilde{\bm w})}
	= \frac{1}{\lambda}\check{X}^\top (\tilde{\bm w} \otimes \bm\alpha^{*(\tilde{\bm w})})$,
we can compute \eqref{eq:safe-sample-screening-rule} and \eqref{eq:duality-gap-l2reg} with kernelized as
\begin{align*}
& \hspace{-1em} \check{X}_{i:}\bm\beta^{*(\tilde{\bm w})} \pm \|\check{X}_{i:}\|_2 R
= \frac{1}{\lambda} \dashuline{\check{X}_{i:}\check{X}^\top} (\tilde{\bm w} \otimes \bm\alpha^{*(\tilde{\bm w})}) \pm \sqrt{\dashuline{\check{X}_{i:}\check{X}_{i:}^\top}} R, \\
& \hspace{-1em} P_{\bm w}(\bm\beta^{*(\tilde{\bm w})}) - D_{\bm w}(\bm\alpha^{*(\tilde{\bm w})}) \\
	&= \sum_{i=1}^n w_i \left[\ell_{y_i}\left(\frac{1}{\lambda} \dashuline{\check{X}_{i:}\check{X}^\top} (\tilde{\bm w} \otimes \bm\alpha^{*(\bm w)})\right) + \ell^*_{y_i}(-\alpha^{*(\tilde{\bm w})}_i)\right] \\
	&\quad + \frac{1}{2\lambda} (\bm w - \tilde{\bm w})^\top \mathrm{diag}(\bm\alpha^{*(\tilde{\bm w})})\dashuline{\check{X} \check{X}^\top} \mathrm{diag}(\bm\alpha^{*(\tilde{\bm w})}) (\bm w - \tilde{\bm w}).
\end{align*}
Here, \dashuline{expressions with dashed lines} denote that the quantities that can be computed
even when kernelized, that is, without using $\bm\beta^{*(\tilde{\bm w})}$ or $X$ except for inner products between samples $X_{i:} X_{i^\prime :}^\top$. This concludes that DRSSS rules can be computed even kernelized with L2-regularization.

\section{NUMERICAL EXPERIMENTS} \label{sec:experiment}


\begin{table}[t]
\caption{Tabular datasets for DRSSS experiments.
All are binary classification datasets from LIBSVM dataset \citep{libsvmDataset}.}
\label{tab:dataset-SS}
\begin{center}
{\small
\begin{tabular}{l|rrr}
\hline
\multicolumn{1}{c|}{Name} & \multicolumn{1}{c}{$n$} & \multicolumn{1}{c}{$n^+$} & \multicolumn{1}{c}{$d$} \\
\hline
australian        &  690 &  307 & 15 \\
breast-cancer     &  683 &  239 & 11 \\
heart             &  270 &  120 & 14 \\
ionosphere        &  351 &  225 & 35 \\
sonar             &  208 &   97 & 61 \\
splice (train)    & 1000 &  517 & 61 \\
svmguide1 (train) & 3089 & 2000 & 5  \\
phishing          &11055 & 6157 & 69 \\
\hline
\end{tabular}
}
\end{center}
\end{table}

We evaluate the performances of DRSSS in the following aspects.
In Section \ref{sec:experiment-weight-changes} and \ref{sec:experiment-deep-learning}
we evaluated the ratios of removed samples by DRSSS.
In Section \ref{sec:experiment-cost} we compared the storage reductions and computational costs
of conventional SSS after DRSSS or without DRSSS.
In Section \ref{sec:experiment-shifts} we compared the ``safeness'' of DRSSS, that is,
DRSSS does not change the model parameters even if a part of samples is removed,
while other strategies do.

All experiments are done with SVM (hinge loss + L2-regularization) model, and applied different kernel functions.
In Section \ref{sec:experiment-weight-changes} we employed the linear kernel (linear prediction) and the \emph{radial basis function (RBF) kernel} \citep{scholkopf2004kernel} for tabular datasets in Table \ref{tab:dataset-SS}.
In Section \ref{sec:experiment-deep-learning} we employed NTK (Section \ref{sec:introduction}) for an image dataset. Since NTK is a kernel to approximate a neural network model, we used such dataset that neural network models are often applied.
In Sections \ref{sec:experiment-cost} and \ref{sec:experiment-shifts}, we examined for the same
tabular datasets and the linear prediction.

The followings are common in all experiments above.
In these experiments, we set the initial sample weights
($\tilde{\bm w}$ in Theorem \ref{thm:safe-screening-robust}) as $\tilde{\bm w} = \bm 1_n$,
and learn $\bm\alpha^{*(\tilde{\bm w})}$ by solving \eqref{eq:dual}.
Since the experiments do not employ validation or test dataset,
we use all samples in the dataset for training.
Then, we set the range of weights as ${\cal W} := \{ \bm w \mid \|\bm w - \tilde{\bm w}\|_2\leq S \}$
(Section \ref{sec:DRSS-examples}), where $S$ is specified as follows:
Since all datasets are for binary classifications,
first we assume the weight change that
	the weights for positive samples ($\{i\mid y_i = +1\}$) from $1$ to $a$,
	while retaining the weights for negative samples ($\{i\mid y_i = -1\}$) as $1$.
Then, we defined $S$ as the size of weight change above; specifically, we set $S = \sqrt{n^+}|a - 1|$ ($n^+$: number of positive samples in the training dataset).
We vary $a$ within the range $0.95\leq a\leq 1.05$, assuming a maximum change of up to 5\% per sample weight.

Please refer Appendix \ref{app:implementation} for details of experimental environments and implementation information.


\subsection{Safe Screening Rates for Tabular Data} \label{sec:experiment-weight-changes}

\begin{figure*}[t]
{\tabcolsep=0mm
\begin{tabular}{ccc}
\begin{minipage}{0.64\hsize}
	{\tabcolsep=0mm
	\begin{tabular}{cc}
	\begin{minipage}{0.5\hsize}
	\centering
	Linear kernel\\
	\includegraphics[width=\hsize]{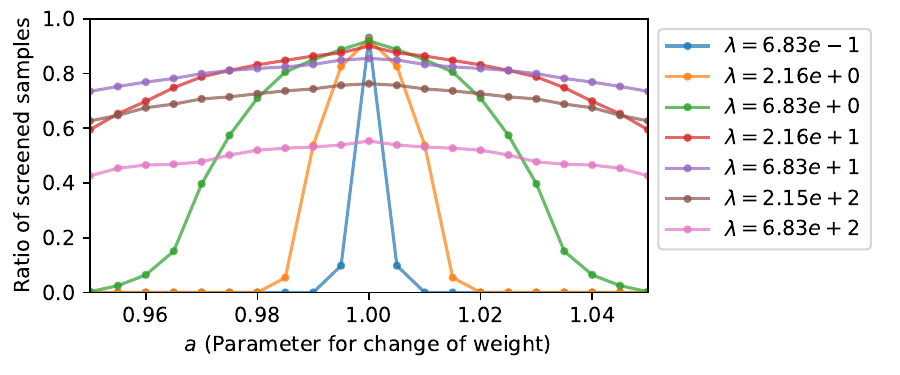}
	\end{minipage}
	&
	\begin{minipage}{0.5\hsize}
	\centering
	RBF kernel\\
	\includegraphics[width=\hsize]{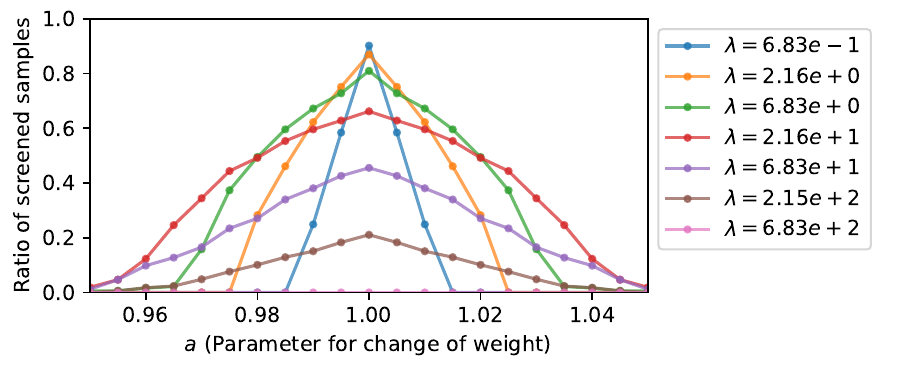}
	\end{minipage}
	\end{tabular}
	}
	\caption{Ratio of removed samples by DRSSS for dataset ``breast-cancer''.}
	\label{fig:SSS-example}
\end{minipage}
&
~~~
&
\begin{minipage}{0.32\hsize}
\centering
\includegraphics[width=\hsize]{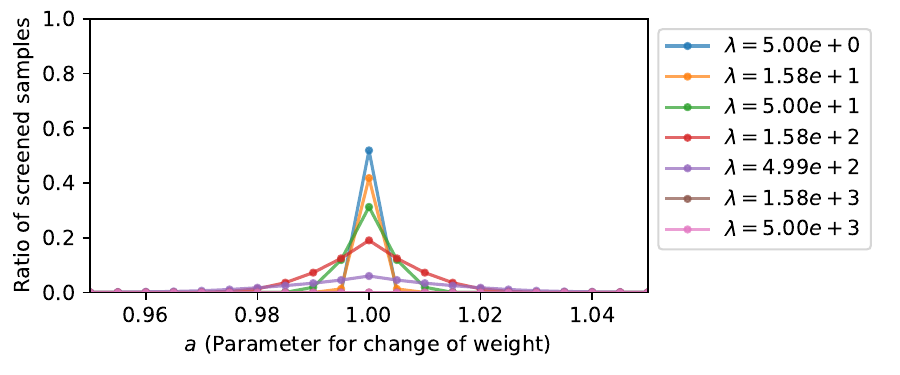}
\caption{Ratio of removed samples by DRSSS for dataset ``CIFAR-10'' using neural tangent kernel (NTK).}
\label{fig:SSS-NTK}
\end{minipage}
\end{tabular}
}
\end{figure*}


First, we present \emph{safe screening rates}, that is, the ratio of the number of removed samples to that of all samples, for two SVM setups: the linear kernel and the RBF kernel.
The RBF kernel is a kernel function for finite-dimensional vectors: given the original dataset $Z\in\mathbb{R}^{n\times d^\prime}$ with $n$ samples and $d^\prime$ features, the RBF kernel can be computed as $X_{i:}X_{i^\prime :}^\top = \exp(-\|Z_{i:} - Z_{i^\prime :}\|_2 / \zeta)$, where $\zeta > 0$ is a hyperparameter.
The datasets used in these experiments are detailed in Table \ref{tab:dataset-SS}.
In this experiment, we adapt the regularization hyperparameter $\lambda$ based on the characteristics of the data.
These details are described in Appendix \ref{app:experimental-setup}.

As an example, for the ``breast-cancer'' dataset, we show the DRSSS result in Figure \ref{fig:SSS-example}
for setups of no kernel and RBF kernel.
Results for other datasets are presented in Appendix \ref{app:experiment}.

These plots allow us to assess the tolerance for changes in sample weights.
Since $a=1$ means no weight change, the screening rates fall as $a$ becomes apart from 1.
We can also see that the screening rates tend to be large for small $\lambda$ if $a$ is near to 1,
while the falls of screening rates is slower for large $\lambda$.
The former means that smaller $\lambda$ tends to have more unnecessary samples ($\alpha^{*(\tilde{\bm w})}_i = 0$) if we do not consider the change of weights.
On the other hand, the behavior of the latter is accounted for by Lemma \ref{lem:gap-sphere-primal}
since larger $\lambda$ reduces
$r(\bm w, \kappa, \hat{\bm\beta}, \hat{\bm\alpha})$.
For instance,
with $a=0.98$ (weight of each positive sample is reduced by two percent, or equivalent weight change in L2-norm),
the sample screening rates are 0.82 (linear) and 0.27 (RBF kernel) for SVM with $\lambda=\mathrm{68.3}$.
This implies that, even if the weights are changed in such range,
a number of samples are still identified as redundant in the sense of prediction.


\subsection{Safe Screening Rates with Neural Tangent Kernel} \label{sec:experiment-deep-learning}


We also experimented DRSSS with NTK for an image dataset.
First, we composed 5,000-sample binary classification dataset from ``CIFAR-10'' dataset \citep{cifar10}, by choosing from classes ``airplane'' (2,529) and `'automobile'' (2,471).
Then, as a neural network for images, we set up a convolutional network of input width 1,024 (same as the number of pixels in an image), three layers of $3 \times 3$ convolution filters, followed by a linear fully connected layer.
With these setups, we trained NTK and obtained the kernel values for all pairs of samples by a Python library ``Neural Tangents'' \citep{neuraltangents2020}.
Finally, similar to tabular datasets, we trained SVM with the NTK and then applied DRSSS.

The result is presented in Figure \ref{fig:SSS-NTK}.
We can confirm a similar tendency to the ones discussed in Section \ref{sec:experiment-weight-changes},
which demonstrates the validity of the proposed method.

\subsection{Effect of DRSSS before Deployment Phase (Multiple Weights)} \label{sec:experiment-cost}

\begin{table*}
\caption{Training sample reductions by DRSSS (\blue{blue}) and the comparison the cost of SSS computations (\red{red}). Note that the values in ``Remained samples'' is the same between the setups ``ISSS'' and ``DRSSS+ISSS''.}
\label{tab:ssevw-cost}
{\small
\begin{tabular}{lrr|rr|rrrr}
\hline
 & & & \multicolumn{2}{c|}{Remained samples} & \multicolumn{4}{c}{Total computation time (s)} \\
 & & & \multicolumn{1}{l}{only} & \multicolumn{1}{l|}{ISSS,} & \multicolumn{1}{c}{ISSS} & \multicolumn{3}{c}{DRSSS+ISSS} \\
\cline{7-9}
\multicolumn{1}{c}{Dataset} & \multicolumn{1}{c}{$n$} & \multicolumn{1}{c|}{$\lambda$} & \multicolumn{1}{l}{DRSSS} & \multicolumn{1}{l|}{DRSSS+ISSS} &  & \multicolumn{1}{c}{Total} & \multicolumn{1}{c}{{\scriptsize DRSSS}} & \multicolumn{1}{c}{{\scriptsize ISSS after DRSSS}} \\
\hline
sonar         &  208 &  65.8 & \blue{75.0\%} & 64.8\% & \red{  1.535} & \red{ 1.204} & {\scriptsize 0.025} & {\scriptsize  1.179} \\
heart         &  270 &  27.0 & \blue{53.3\%} & 47.7\% & \red{  1.180} & \red{ 0.720} & {\scriptsize 0.008} & {\scriptsize  0.712} \\
ionosphere    &  351 & 111.0 & \blue{78.6\%} & 61.4\% & \red{  4.337} & \red{ 2.242} & {\scriptsize 0.018} & {\scriptsize  2.223} \\
breast-cancer &  683 &  21.6 & \blue{13.6\%} & 11.0\% & \red{  4.754} & \red{ 0.586} & {\scriptsize 0.010} & {\scriptsize  0.577} \\
australian    &  690 & 218.2 & \blue{65.7\%} & 55.8\% & \red{ 13.517} & \red{ 2.326} & {\scriptsize 0.016} & {\scriptsize  2.311} \\
splice        & 1000 & 316.2 & \blue{79.0\%} & 68.8\% & \red{ 17.216} & \red{11.256} & {\scriptsize 0.077} & {\scriptsize 11.179} \\
svmguide1     & 3089 &  97.7 & \blue{57.6\%} & 48.4\% & \red{169.767} & \red{22.645} & {\scriptsize 0.148} & {\scriptsize 22.497} \\
phishing      &11055 &3495.9 & \blue{58.1\%} & 36.3\% &\red{3457.354} &\red{1193.231}&{\scriptsize 11.279} & {\scriptsize 1181.952} \\
\hline
\multicolumn{4}{r}{Sample reduction in storage $\uparrow$} & \multicolumn{5}{l}{$\uparrow$ Sample reduction when training} \\
\multicolumn{4}{r}{(Reduction \emph{before} weight changes)} & \multicolumn{5}{l}{(Reduction \emph{after} weight change determined)}
\end{tabular}
}

\caption{Model parameter shifts in L2-norm of $\bm\beta^*$ by sample removal strategies, for 100 random choices of $\bm w$ (except dataset ``phishing''). Note that, since the model parameter computations are done with 32-bit floating-point values, shifts less than $10^{-6}$ are considered to include only numerical errors.}
\label{tab:proposed-vs-naive-removal}
{\small
\begin{tabular}{lrcc|ccc}
\hline
 & & \multicolumn{2}{c|}{Remained samples} & \multicolumn{3}{c}{Model parameter shifts (Average$\pm$StDev)} \\
\multicolumn{1}{c}{Dataset} & \multicolumn{1}{c}{$n$} & DRSSS, Random & NaiveSS & DRSSS & Random & NaiveSS \\
\hline
sonar         &  208 & 78.40\% & 71.60\% & 5.03e-07 $\pm$ 4.31e-07 & 1.24e-01 $\pm$ 1.10e-04 & 4.86e-02 $\pm$ 3.54e-04 \\
heart         &  270 & 74.80\% & 72.60\% & 9.17e-08 $\pm$ 1.03e-07 & 8.19e-02 $\pm$ 1.21e-04 & 7.60e-03 $\pm$ 3.85e-04 \\
ionosphere    &  351 & 82.10\% & 76.40\% & 2.38e-07 $\pm$ 2.45e-07 & 8.76e-02 $\pm$ 1.43e-04 & 3.41e-02 $\pm$ 3.89e-04 \\
breast-cancer &  683 & 46.90\% & 44.80\% & 1.76e-08 $\pm$ 4.55e-08 & 9.66e-02 $\pm$ 6.63e-05 & 2.63e-06 $\pm$ 1.50e-05 \\
australian    &  690 & 78.10\% & 75.70\% & 1.69e-07 $\pm$ 1.72e-07 & 7.19e-02 $\pm$ 5.80e-05 & 4.11e-03 $\pm$ 2.10e-04 \\
splice        & 1000 & 88.00\% & 84.50\% & 1.05e-07 $\pm$ 9.79e-08 & 6.50e-02 $\pm$ 7.48e-05 & 1.05e-02 $\pm$ 2.17e-04 \\
svmguide1     & 3089 & 85.10\% & 84.20\% & 8.08e-08 $\pm$ 3.28e-07 & 3.42e-02 $\pm$ 4.35e-05 & 7.34e-08 $\pm$ 3.06e-07 \\
\hline
\end{tabular}
}
\end{table*}

The proposed method assumes that the sample weights change in a certain range,
and aims to remove unnecessary samples for any weight in the range.
We examine the effect of storage reductions of training samples by generating weights at random.
%
%

For the weight change $S$ at $a=0.95$, we generated $\bm w$ for 10,000 times at random (details in Appendix \ref{app:random-weights}).
With the setups above, we demonstrate the training sample reduction and compare the cost of SSS computations
for the following two strategies.
One is ISSS (individual safe sample screening): For each of randomly chosen weights, just apply non-DR SSS.
The other is DRSSS+ISSS: First we apply DRSSS to remove the training samples. Then, for each of randomly chosen weights, we apply non-DR SSS for the remained samples.
Note that we omit the training cost between these two strategies since DRSSS always removes a subset of samples removed by non-DR SSS. As a result, both of the methods above keep the same set of samples before training.

The results are presented in Table \ref{tab:ssevw-cost}.
Firstly, only using non-DR SSS, we could not remove samples before the weights are provided.
On the other hand, DRSSS can reduce a part of samples for the ratio (see the column ``Remained samples -- DRSSS'').
Comparing the cost for safe screening, applying DRSSS can reduce the computational cost.
From the result, we can see that the cost of SSS is reduced more rapidly compared to the reduction of the
training samples.
This is because the cost of non-DR SSS is two-fold against the number of samples:
For $n$-sample training data, we need to examine SSS rules for $n$ times,
and we need to compute SSS rules by using an $n$-dimensional vector $\hat{\bm\alpha}$ (Lemma \ref{lem:gap-sphere-primal}).
This concludes that, when we need to train a model for various weights,
DRSSS is effective both in the storage and the computational cost.

\subsection{Model Parameter Shifts by DRSSS and Other Sample Removals} \label{sec:experiment-shifts}

The proposed method DRSSS is a method to remove samples that do not change the trained model parameters.
So we demonstrate other methods of sample removals in the change of the model parameters.
For DRSSS, we remove samples by Theorem \ref{thm:safe-screening-robust}.
%
We also examine the following two methods.
One is ``Random'': removing the same number of samples as that removed by DRSSS at random.
The other is ``Naive SS'': removing samples by non-DR SSS (Lemma \ref{lem:gap-sample-screening}),
which may remove a larger number of samples to that of DRSSS.

The change of the model parameters are examined as follows:
we $S$ as the case of $a=0.99$, and $\lambda$ as $n$.
After samples are removed for each of the setups, for 100 randomly generated weights (details in Appendix \ref{app:random-weights}) we train the model parameters and measure the difference in L2-norm of the primal model parameter $\bm\beta^{*(\bm w)}$ before and after sample removals.
(For ``random removal'' setup, we examined 10 cases of removals and took the average.)

We present results in Table \ref{tab:proposed-vs-naive-removal}.
First, we can confirm that the difference of model parameters after DRSSS causes only numerical errors.
On the other hand, even the same number of samples are removed,
``Random'' causes much larger difference in the model parameters.
The method ``NaiveSS'' causes smaller difference compared to ``Random'', however,
it is much larger than ``DRSSS''.
This concludes that sample removals by DRSSS can keep the model parameters unchanged,
and under the precondition it can remove appropriate number of samples.

\section{CONCLUSIONS, LIMITATIONS AND FUTURE WORKS} \label{sec:conclusion}
%
In this study, we propose the DRSSS method as a robust data reduction method for situations where the data distribution at deployment is unknown.
This method effectively reduces data storage and model update costs in DR learning settings.
Our technical contribution extends the SSS technique, originally developed for accelerating a single optimization problem, to be applicable across an entire class of optimization problems within a specific category.
The current limitation of the proposed method is that the applicability of DRSSS, similar to SSS, is confined to convex optimization problems of a specific class.
Therefore, applying it to complex models such as deep learning requires using kernel methods based on NTK, meaning it cannot be directly applied to deep learning models.
Additionally, there may be cases where the removable data volume by DRSSS is insufficient,
since DRSSS identifies unnecessary samples under a stringent condition that exactly the same model parameter should be obtained even when they are removed.
It is, therefore, crucial to develop a framework that relaxes these requirements and allows for the removal of more data.

\subsubsection*{Acknowledgements}

This work was partially supported by MEXT KAKENHI (JP20H00601, JP23K16943, JP24K15080), JST CREST (JPMJCR21D3 including AIP challenge program, JPMJCR22N2), JST Moonshot R\&D (JPMJMS2033-05), JST AIP Acceleration Research (JPMJCR21U2), NEDO (JPNP18002, JPNP20006) and RIKEN Center for Advanced Intelligence Project.

\bibliography{paper}

\begin{thebibliography}{}

\bibitem[Chang and Lin, 2011]{libsvmDataset}
Chang, C.-C. and Lin, C.-J. (2011).
\newblock Libsvm: A library for support vector machines.
\newblock {\em ACM Transactions on Intelligent Systems and Technology (TIST)},
  2(3):27.
\newblock Datasets are provided in authors' website:
  \url{https://www.csie.ntu.edu.tw/~cjlin/libsvmtools/datasets/}.

\bibitem[Chen and Paschalidis, 2021]{chen2021distributionally}
Chen, R. and Paschalidis, I.~C. (2021).
\newblock Distributionally robust learning.
\newblock arXiv Preprint.

\bibitem[Delage and Ye, 2010]{delage2010distributionally}
Delage, E. and Ye, Y. (2010).
\newblock Distributionally robust optimization under moment uncertainty with
  application to data-driven problems.
\newblock {\em Operations Research}, 58(3):595--612.

\bibitem[Diamond and Boyd, 2016]{diamond2016cvxpy}
Diamond, S. and Boyd, S. (2016).
\newblock {CVXPY}: A {P}ython-embedded modeling language for convex
  optimization.
\newblock {\em Journal of Machine Learning Research}.
\newblock To appear.

\bibitem[El~Ghaoui et~al., 2012]{ghaoui2012safe}
El~Ghaoui, L., Viallon, V., and Rabbani, T. (2012).
\newblock Safe feature elimination for the lasso and sparse supervised learning
  problems.
\newblock {\em Pacific Journal of Optimization}, 8(4):667--698.

\bibitem[Fercoq et~al., 2015]{fercoq2015mind}
Fercoq, O., Gramfort, A., and Salmon, J. (2015).
\newblock Mind the duality gap: safer rules for the lasso.
\newblock In {\em Proceedings of the 32nd International Conference on Machine
  Learning}, pages 333--342.

\bibitem[Geng et~al., 2023]{geng2023dataset}
Geng, J., Chen, Z., Wang, Y., Woisetschl\"{a}ger, H., Schimmler, S., Mayer, R.,
  Zhao, Z., and Rong, C. (2023).
\newblock A survey on dataset distillation: approaches, applications and future
  directions.
\newblock In {\em Proceedings of the Thirty-Second International Joint
  Conference on Artificial Intelligence}.

\bibitem[Goh and Sim, 2010]{goh2010distributionally}
Goh, J. and Sim, M. (2010).
\newblock Distributionally robust optimization and its tractable
  approximations.
\newblock {\em Operations Research}, 58(4-1):902--917.

\bibitem[Harris et~al., 2020]{harris2020array}
Harris, C.~R., Millman, K.~J., van~der Walt, S.~J., Gommers, R., Virtanen, P.,
  Cournapeau, D., Wieser, E., Taylor, J., Berg, S., Smith, N.~J., Kern, R.,
  Picus, M., Hoyer, S., van Kerkwijk, M.~H., Brett, M., Haldane, A., del
  R{\'{i}}o, J.~F., Wiebe, M., Peterson, P., G{\'{e}}rard-Marchant, P.,
  Sheppard, K., Reddy, T., Weckesser, W., Abbasi, H., Gohlke, C., and Oliphant,
  T.~E. (2020).
\newblock Array programming with {NumPy}.
\newblock {\em Nature}, 585(7825):357--362.

\bibitem[Hiriart-Urruty and Lemar{\'e}chal, 1993]{hiriart1993convex}
Hiriart-Urruty, J.-B. and Lemar{\'e}chal, C. (1993).
\newblock {\em Convex Analysis and Minimization Algorithms II: Advanced Theory
  and Bundle Methods}.
\newblock Springer.

\bibitem[Jacot et~al., 2018]{jacot2018neural}
Jacot, A., Gabriel, F., and Hongler, C. (2018).
\newblock Neural tangent kernel: Convergence and generalization in neural
  networks.
\newblock {\em Advances in neural information processing systems}, 31.

\bibitem[Kirkpatrick et~al., 2017]{kirkpatrick2017overcoming}
Kirkpatrick, J., Pascanu, R., Rabinowitz, N., Veness, J., Desjardins, G., Rusu,
  A.~A., Milan, K., Quan, J., Ramalho, T., Grabska-Barwinska, A., Hassabis, D.,
  Clopath, C., Kumaran, D., and Hadsell, R. (2017).
\newblock Overcoming catastrophic forgetting in neural networks.
\newblock {\em Proceedings of the National Academy of Sciences},
  114(13):3521--3526.

\bibitem[Krizhevsky, 2009]{cifar10}
Krizhevsky, A. (2009).
\newblock The cifar-10 dataset.

\bibitem[Lei and Tao, 2024]{lei2024comprehensive}
Lei, S. and Tao, D. (2024).
\newblock A comprehensive survey of dataset distillation.
\newblock {\em IEEE Transactions on Pattern Analysis and Machine Intelligence},
  46(1):17--32.

\bibitem[Mirzasoleiman et~al., 2020]{mirzasoleiman2020coresets}
Mirzasoleiman, B., Bilmes, J., and Leskovec, J. (2020).
\newblock Coresets for data-efficient training of machine learning models.
\newblock In {\em Proceedings of the 37th International Conference on Machine
  Learning}, pages 6950--6960.

\bibitem[Nakagawa et~al., 2016]{nakagawa2016safe}
Nakagawa, K., Suzumura, S., Karasuyama, M., Tsuda, K., and Takeuchi, I. (2016).
\newblock Safe pattern pruning: An efficient approach for predictive pattern
  mining.
\newblock In {\em Proceedings of the 22nd ACM SIGKDD International Conference
  on Knowledge Discovery and Data Mining}, pages 1785--1794. ACM.

\bibitem[Ndiaye et~al., 2015]{ndiaye2015gap}
Ndiaye, E., Fercoq, O., Gramfort, A., and Salmon, J. (2015).
\newblock Gap safe screening rules for sparse multi-task and multi-class
  models.
\newblock In {\em Advances in Neural Information Processing Systems}, pages
  811--819.

\bibitem[Novak et~al., 2020]{neuraltangents2020}
Novak, R., Xiao, L., Hron, J., Lee, J., Alemi, A.~A., Sohl-Dickstein, J., and
  Schoenholz, S.~S. (2020).
\newblock Neural tangents: Fast and easy infinite neural networks in python.
\newblock In {\em International Conference on Learning Representations}.

\bibitem[Ogawa et~al., 2013]{ogawa2013safe}
Ogawa, K., Suzuki, Y., and Takeuchi, I. (2013).
\newblock Safe screening of non-support vectors in pathwise svm computation.
\newblock In {\em Proceedings of the 30th International Conference on Machine
  Learning}, pages 1382--1390.

\bibitem[Okumura et~al., 2015]{okumura2015quick}
Okumura, S., Suzuki, Y., and Takeuchi, I. (2015).
\newblock Quick sensitivity analysis for incremental data modification and its
  application to leave-one-out cv in linear classification problems.
\newblock In {\em Proceedings of the 21th ACM SIGKDD International Conference
  on Knowledge Discovery and Data Mining}, pages 885--894.

\bibitem[Paszke et~al., 2017]{paszke2017automatic}
Paszke, A., Gross, S., Chintala, S., Chanan, G., Yang, E., DeVito, Z., Lin, Z.,
  Desmaison, A., Antiga, L., and Lerer, A. (2017).
\newblock Automatic differentiation in pytorch.
\newblock NIPS 2017 Workshop Autodiff.

\bibitem[Pedregosa et~al., 2011]{scikit-learn}
Pedregosa, F., Varoquaux, G., Gramfort, A., Michel, V., Thirion, B., Grisel,
  O., Blondel, M., Prettenhofer, P., Weiss, R., Dubourg, V., Vanderplas, J.,
  Passos, A., Cournapeau, D., Brucher, M., Perrot, M., and Duchesnay, E.
  (2011).
\newblock Scikit-learn: Machine learning in {P}ython.
\newblock {\em Journal of Machine Learning Research}, 12:2825--2830.

\bibitem[Ren et~al., 2018]{ren2018safe}
Ren, S., Huang, S., Ye, J., and Qian, X. (2018).
\newblock Safe feature screening for generalized lasso.
\newblock {\em IEEE Transactions on Pattern Analysis and Machine Intelligence},
  40(12):2992--3006.

\bibitem[Rockafellar, 1970]{rockafellar1970convex}
Rockafellar, R.~T. (1970).
\newblock {\em Convex analysis}.
\newblock Princeton university press.

\bibitem[Sch{\"o}lkopf et~al., 2001]{scholkopf2001generalized}
Sch{\"o}lkopf, B., Herbrich, R., and Smola, A.~J. (2001).
\newblock A generalized representer theorem.
\newblock In {\em International conference on computational learning theory},
  pages 416--426. Springer.

\bibitem[Sch{\"o}lkopf et~al., 2004]{scholkopf2004kernel}
Sch{\"o}lkopf, B., Tsuda, K., and Vert, J., editors (2004).
\newblock {\em Kernel Methods in Computational Biology}.
\newblock The MIT Press.

\bibitem[Sener and Savarese, 2018]{sener2018active}
Sener, O. and Savarese, S. (2018).
\newblock Active learning for convolutional neural networks: A core-set
  approach.
\newblock In {\em The 6th International Conference on Learning Representations
  (ICLR)}.

\bibitem[Shibagaki et~al., 2016]{shibagaki2016simultaneous}
Shibagaki, A., Karasuyama, M., Hatano, K., and Takeuchi, I. (2016).
\newblock Simultaneous safe screening of features and samples in doubly sparse
  modeling.
\newblock In {\em International Conference on Machine Learning}, pages
  1577--1586.

\bibitem[Shimodaira, 2000]{shimodaira2000improving}
Shimodaira, H. (2000).
\newblock Improving predictive inference under covariate shift by weighting the
  log-likelihood function.
\newblock {\em Journal of statistical planning and inference}, 90(2):227--244.

\bibitem[Sugiyama et~al., 2007]{sugiyama2007covariate}
Sugiyama, M., Krauledat, M., and M{{\"u}}ller, K.-R. (2007).
\newblock Covariate shift adaptation by importance weighted cross validation.
\newblock {\em Journal of Machine Learning Research}, 8(35):985--1005.

\bibitem[Virtanen et~al., 2020]{2020SciPy-NMeth}
Virtanen, P., Gommers, R., Oliphant, T.~E., Haberland, M., Reddy, T.,
  Cournapeau, D., Burovski, E., Peterson, P., Weckesser, W., Bright, J., {van
  der Walt}, S.~J., Brett, M., Wilson, J., Millman, K.~J., Mayorov, N., Nelson,
  A. R.~J., Jones, E., Kern, R., Larson, E., Carey, C.~J., Polat, {\.I}., Feng,
  Y., Moore, E.~W., {VanderPlas}, J., Laxalde, D., Perktold, J., Cimrman, R.,
  Henriksen, I., Quintero, E.~A., Harris, C.~R., Archibald, A.~M., Ribeiro,
  A.~H., Pedregosa, F., {van Mulbregt}, P., and {SciPy 1.0 Contributors}
  (2020).
\newblock {{SciPy} 1.0: Fundamental Algorithms for Scientific Computing in
  Python}.
\newblock {\em Nature Methods}, 17:261--272.

\bibitem[Wang and Xu, 2022]{wang2022safe}
Wang, H. and Xu, Y. (2022).
\newblock A safe double screening strategy for elastic net support vector
  machine.
\newblock {\em Information Sciences}, 582:382--397.

\bibitem[Wang et~al., 2022]{wang2022memory}
Wang, L., Zhang, X., Yang, K., Yu, L., Li, C., HONG, L., Zhang, S., Li, Z.,
  Zhong, Y., and Zhu, J. (2022).
\newblock Memory replay with data compression for continual learning.
\newblock In {\em International Conference on Learning Representations}.

\bibitem[Yoshida et~al., 2023]{yoshida2023efficient}
Yoshida, T., Hanada, H., Nakagawa, K., Taji, K., Tsuda, K., and Takeuchi, I.
  (2023).
\newblock Efficient model selection for predictive pattern mining model by safe
  pattern pruning.
\newblock {\em Patterns}, 4(12):100890.

\bibitem[Zhai et~al., 2020]{zhai2020safe}
Zhai, Z., Gu, B., Li, X., and Huang, H. (2020).
\newblock Safe sample screening for robust support vector machine.
\newblock In {\em Proceedings of the AAAI Conference on Artificial
  Intelligence}, volume~34, pages 6981--6988.

\bibitem[Zhao et~al., 2019]{zhao2019improved}
Zhao, J., Xu, Y., and Fujita, H. (2019).
\newblock An improved non-parallel universum support vector machine and its
  safe sample screening rule.
\newblock {\em Knowledge-Based Systems}, 170:79--88.

\bibitem[Zou and Hastie, 2005]{zou05regularization}
Zou, H. and Hastie, T. (2005).
\newblock Regularization and variable selection via the elastic net.
\newblock {\em Journal of the Royal Statistical Society, Series B},
  67:301--320.

\end{thebibliography}
\bibliographystyle{apalike}

%
%
\clearpage
\onecolumn
\appendix

\section{PROOFS OF GENERAL LEMMAS} \label{app:general-lemmas}

In this appendix we show several lemmas used in other appendices.

\begin{lemma}\label{lem:fenchel-moreau} \emph{(Fenchel-Moreau theorem)}
For a convex function $f: \mathbb{R}^d\to\mathbb{R}\cup\{+\infty\}$,
$f^{**}$ is equivalent to $f$ if $f$ is convex, proper (i.e., $\exists \bm v\in\mathbb{R}^d:~f(\bm v) < +\infty$) and lower-semicontinuous.
\end{lemma}

\begin{proof}
See Section 12 of \citep{rockafellar1970convex} for example.
\end{proof}

As a special case of Lemma \ref{lem:fenchel-moreau}, it holds if $f$ is finite ($\forall \bm v\in\mathbb{R}^d:~f(\bm v)<+\infty$) and convex.

\begin{lemma}\label{lem:strong-convexity-smoothness}
For a convex function $f: \mathbb{R}^d\to\mathbb{R}\cup\{+\infty\}$,
\begin{itemize}
\item $f^*$ is $(1/\nu)$-strongly convex if $f$ is proper and $\nu$-smooth.
\item $f^*$ is $(1/\kappa)$-smooth if $f$ is proper, lower-semicontinuous and $\kappa$-strongly convex.
\end{itemize}
\end{lemma}

\begin{proof}
See Section X.4.2 of \citep{hiriart1993convex} for example.
\end{proof}

\begin{corollary} \label{cor:strong-convexity-smoothness}
Suppose that the regularization function $\rho$ in \eqref{eq:primal} is $\kappa$-strongly convex, which is required to apply SSS (Lemma \ref{lem:gap-sphere-primal}).
Then, from Lemma \ref{lem:strong-convexity-smoothness}, $\rho^*$ must be $(1/\kappa)$-smooth. This implies that, in this case, $\rho^*$ cannot be infinite.
\end{corollary}

\begin{lemma}\label{lem:strong-convexity-sphere}
Suppose that $f: \mathbb{R}^d\to\mathbb{R}\cup\{+\infty\}$ is a $\kappa$-strongly convex function,
and let $\bm v^* = \targmin_{\bm v\in\mathbb{R}^d} f(\bm v)$ be the minimizer of $f$.
Then, for any $\bm v\in\mathbb{R}^d$, we have
\begin{align*}
\| \bm v - \bm v^* \|_2 \leq \sqrt{\frac{2}{\kappa}[f(\bm v) - f(\bm v^*)]}.
\end{align*}
\end{lemma}

\begin{proof}
See \citep{ndiaye2015gap} for example.
\end{proof}

\begin{lemma} \label{lem:optimize-linear}
For any vector $\bm a, \bm c\in\mathbb{R}^n$ and $S > 0$,
\begin{align*}
& \min_{\bm v\in\mathbb{R}^n:~\|\bm v - \bm c\|_2\leq S} \bm a^\top \bm v = \bm a^\top \bm c - S\|\bm a\|_2,
& \max_{\bm v\in\mathbb{R}^n:~\|\bm v - \bm c\|_2\leq S} \bm a^\top \bm v = \bm a^\top \bm c + S\|\bm a\|_2. \\
\end{align*}
\end{lemma}

\begin{proof}
By Cauchy-Schwarz inequality,
\begin{align*}
& -\|\bm a\|_2 \|\bm v - \bm c\|_2 \leq \bm a^\top (\bm v - \bm c) \leq \|\bm a\|_2 \|\bm v - \bm c\|_2.
\end{align*}
Noticing that the first inequality becomes equality if $\exists\omega>0:~\bm a = -\omega(\bm v - \bm c)$,
while the second inequality becomes equality if $\exists\omega^\prime>0:~\bm a = \omega^\prime(\bm v - \bm c)$.
Moreover, since $\|\bm v - \bm c\|_2\leq S$,
\begin{align*}
& - S \|\bm a\|_2 \leq \bm a^\top (\bm v - \bm c) \leq S \|\bm a\|_2
\end{align*}
also holds, with the equality holds if $\|\bm v - \bm c\|_2 = S$.

On the other hand, if we take $\bm v$ that satisfies both of the equality conditions
of Cauchy-Schwarz inequality above, that is,
\begin{itemize}
\item (for the first inequality being equality) $\bm v = \bm c - (S/\|\bm a\|_2)\bm a$,
\item (for the second inequality being equality) $\bm v = \bm c + (S/\|\bm a\|_2)\bm a$,
\end{itemize}
then the inequalities become equalities. This proves that $- S \|\bm a\|_2$ and $S \|\bm a\|_2$ are surely the minimum and maximum of $\bm a^\top (\bm v - \bm c)$, respectively.
\end{proof}

\section{PROOFS AND NOTES OF SECTION \ref{sec:preliminaries}}

\subsection{Derivation of Dual Problem by Fenchel's Duality Theorem} \label{app:fenchel}

As the formulation of Fenchel's duality theorem, we follow the one in Section 31 of \citep{rockafellar1970convex}.

\begin{lemma}[A special case of Fenchel's duality theorem: $f, g<+\infty$] \label{lm:Fenchel-duality}
Let $f: \mathbb{R}^n\to\mathbb{R}$ and $g: \mathbb{R}^d\to\mathbb{R}$ be convex functions,
and $A\in\mathbb{R}^{n\times d}$ be a matrix.
Moreover, we define
\begin{align}
& \bm v^* := \min_{\bm v\in\mathbb{R}^d} [f(A\bm v) + g(\bm v)], \label{eq:FD-primal} \\
& \bm u^* := \max_{\bm u\in\mathbb{R}^n} [ -f^*(-\bm u) - g^*(A^\top\bm u)]. \label{eq:FD-dual}
\end{align}
Then Fenchel's duality theorem assures that
\begin{align*}
& f(A\bm v^*) + g(\bm v^*) = -f^*(-\bm u^*) - g^*(A^\top\bm u^*), \\
& -\bm u^* \in \partial f(A \bm v^*), \\
& \bm v^* \in \partial g^*(A^\top \bm u^*).
\end{align*}
\end{lemma}

\begin{proof}
Introducing a dummy variable $\bm\psi\in\mathbb{R}^n$ and a Lagrange multiplier $\bm u\in\mathbb{R}^n$, we have
\begin{align}
& \min_{\bm v\in\mathbb{R}^d} [f(A\bm v) + g(\bm v)]
	= \max_{\bm u\in\mathbb{R}^n} \min_{\bm v\in\mathbb{R}^d,~\bm\psi\in\mathbb{R}^n} [f(\bm\psi) + g(\bm v) - \bm u^\top(A\bm v - \bm\psi)] \label{eq:Fenchel-Lagrange}\\
& = - \min_{\bm u\in\mathbb{R}^n} \max_{\bm v\in\mathbb{R}^d,~\bm\psi\in\mathbb{R}^n} [-f(\bm\psi) - g(\bm v) + \bm u^\top(A\bm v - \bm\psi)] \nonumber\\
& = - \min_{\bm u\in\mathbb{R}^n} \max_{\bm v\in\mathbb{R}^d,~\bm\psi\in\mathbb{R}^n} [\{ (-\bm u)^\top \bm\psi - f(\bm\psi) \} + \{ (A^\top \bm u)^\top \bm v - g(\bm v) \}] \nonumber\\
& = - \min_{\bm u\in\mathbb{R}^n} [f^*(-\bm u) + g^*(A^\top \bm u)]
	= \max_{\bm u\in\mathbb{R}^n} [-f^*(-\bm u) - g^*(A^\top \bm u)]. \label{eq:Fenchel-dual}
\end{align}
Moreover, by the optimality condition of a problem with a Lagrange multiplier \eqref{eq:Fenchel-Lagrange},
the optima of it, denoted by $\bm v^*$, $\bm\psi^*$ and $\bm u^*$, must satisfy
\begin{align*}
A \bm v^* = \bm\psi^*,
\quad
A^\top \bm u^* \in\partial g(\bm v^*),
\quad
-\bm u^* \in \partial f(\bm\psi^*) = \partial f(A \bm v^*).
\end{align*}
On the other hand, introducing a dummy variable $\bm\phi\in\mathbb{R}^d$ and a Lagrange multiplier $\bm v\in\mathbb{R}^d$ for \eqref{eq:Fenchel-dual}, we have
\begin{align}
& \max_{\bm u\in\mathbb{R}^n} [-f^*(-\bm u) - g^*(A^\top \bm u)]
	= \min_{\bm v\in\mathbb{R}^d} \max_{\bm u\in\mathbb{R}^n, \bm\phi\in\mathbb{R}^d} [-f^*(-\bm u) - g^*(\bm\phi) - \bm v^\top(A^\top \bm u - \bm\phi)] \label{eq:Fenchel-dual-Lagrange}\\
& = \min_{\bm v\in\mathbb{R}^d} \max_{\bm u\in\mathbb{R}^n, \bm\phi\in\mathbb{R}^d} [\{(A\bm v)^\top(-\bm u) - f^*(-\bm u)\} + \{\bm v^\top \bm\phi - g^*(\bm\phi)\}] \nonumber \\
& = \min_{\bm v\in\mathbb{R}^d} [f^{**}(A\bm v) + g^{**}(\bm v)]
	= \min_{\bm v\in\mathbb{R}^d} [f(A\bm v) + g(\bm v)]. \quad(\because~\text{Lemma \ref{lem:fenchel-moreau}})\nonumber
\end{align}
Likely above, by the optimality condition of a problem with a Lagrange multiplier \eqref{eq:Fenchel-dual-Lagrange},
the optima of it, denoted by $\bm u^*$, $\bm\phi^*$ and $\bm v^*$, must satisfy
\begin{align*}
A^\top \bm u^* = \bm\phi^*,
\quad
\bm v^* \in \partial g^*(\bm\phi^*) = \partial g^*(A^\top \bm u^*),
\quad
A \bm v^* \in \partial f(-\bm u^*).
\end{align*}
\end{proof}

\begin{lemma}[Dual problem of weighted regularized empirical risk minimization (weighted RERM)] \label{lm:dual-WRERM}
For the minimization problem
\begin{align}
& \bm\beta^{*(\bm w)} := \argmin_{\bm\beta\in\mathbb{R}^d} P_{\bm w}(\bm\beta), \nonumber
	\quad
	\text{where}
	\quad
	P_{\bm w}(\bm\beta) := \sum_{i=1}^n w_i \ell_{y_i}(\check{X}_{i:}\bm\beta) + \rho(\bm\beta), \tag{\eqref{eq:primal} restated}
\end{align}
we define the dual problem as the one obtained by applying Fenchel's duality theorem (Lemma \ref{lm:Fenchel-duality}), which is defined as
\begin{align}
& \bm\alpha^{*(\bm w)} := \argmax_{\bm\alpha\in\mathbb{R}^n} D_{\bm w}(\bm\alpha), \nonumber
	\quad
	\text{where}
	\quad
	D_{\bm w}(\bm\alpha) := -\sum_{i=1}^n w_i \ell^*_{y_i}(-\alpha_i) + \rho^*((\bm w\sqtimes\check{X})^\top \bm\alpha). \tag{\eqref{eq:dual} restated}
\end{align}
Moreover, $\bm\beta^{*(\bm w)}$ and $\bm\alpha^{*(\bm w)}$ must satisfy
\begin{align}
& P_{\bm w}(\bm\beta^{*(\bm w)}) = D_{\bm w}(\bm\alpha^{*(\bm w)}), \tag{\eqref{eq:strong-duality} restated}\\
& \bm\beta^{*(\bm w)} \in \partial\rho^*((\bm w\sqtimes\check{X})^\top \bm\alpha^{*(\bm w)}), \tag{\eqref{eq:KKT-dual2primal} restated}\\
& \forall i\in[n]:\quad -\alpha^{*(\bm w)}_i \in \partial\ell_{y_i}(\check{X}_{i:}\bm\beta^{*(\bm w)}). \tag{\eqref{eq:KKT-primal2dual} restated}
\end{align}
\end{lemma}

\begin{proof}
To apply Fenchel's duality theorem, we have only to set $f$, $g$ and $A$ in Lemma \ref{lm:Fenchel-duality} as
\begin{align*}
f(\bm u) := \sum_{i=1}^n w_i \ell_{y_i}(u_i),
\quad
g(\bm\beta) := \rho(\bm\beta),
\quad
A := \check{X}.
\end{align*}
Here, noticing that
\begin{align*}
& f^*(\bm u)
	= \sup_{\bm u^\prime\in\mathbb{R}^n} [\bm u^\top \bm u^\prime - \sum_{i=1}^n w_i \ell_{y_i}(u^\prime_i)]
	= \sup_{\bm u^\prime\in\mathbb{R}^n} \sum_{i=1}^n \left[u_i u^\prime_i - w_i \ell_{y_i}(u^\prime_i) \right] \\
& = \sup_{\bm u^\prime\in\mathbb{R}^n} \sum_{i=1}^n w_i \left[\frac{u_i}{w_i} u^\prime_i - \ell_{y_i}(u^\prime_i) \right]
	= \sum_{i=1}^n w_i \ell_{y_i}^*\left(\frac{u_i}{w_i}\right),
\end{align*}
from \eqref{eq:FD-dual} we have
\begin{align*}
-f^*(-\bm u) - g^*(A^\top\bm u)
= - \sum_{i=1}^n w_i \ell_{y_i}^*\left(-\frac{u_i}{w_i}\right) - \rho^*(\check{X}^\top\bm u).
\end{align*}
Replacing $u_i\gets w_i \alpha_i$, that is, $\bm u\gets (\bm w \otimes \bm\alpha)$,
we have the dual problem \eqref{eq:dual}.

The relationships between the primal and the dual problem are described as follows:
\begin{align*}
& -\bm u^* \in \partial f(A \bm v^*)
	~\Rightarrow~
	-\bm w \otimes \bm\alpha^{*(\bm w)} \in \partial f(\check{X} \bm\beta^{*(\bm w)})
	~\Rightarrow~
	-w_i \alpha^{*(\bm w)}_i \in w_i \partial\ell_{y_i}(\check{X}_{i:}\bm\beta^{*(\bm w)}) \\
	& \Rightarrow
	-\alpha^{*(\bm w)}_i \in \partial\ell_{y_i}(\check{X}_{i:}\bm\beta^{*(\bm w)}), \\
& \bm v^* \in \partial g^*(A^\top \bm u^*)
	~\Rightarrow~
	\bm\beta^{*(\bm w)} \in \partial g^*(\check{X}^\top \bm w \otimes \bm\alpha^{*(\bm w)})
	= \partial g^*((\bm w\sqtimes\check{X})^\top \bm\alpha^{*(\bm w)}).
\end{align*}
\end{proof}

\subsection{Examples of Sample-Sparse Loss Functions} \label{app:loss-functions}

We presented examples of sample-sparse loss functions in Table \ref{tab:loss-functions}.

We should notice that, depending on $\bm\alpha$, $D_{\bm w}(\bm\alpha)$ may be infinite
and therefore we may not use such $\bm\alpha$ when computing it.
This is because the convex conjugate $f^*$ (Table \ref{tab:definitions}) may be infinite even $f$ is finite.
In this setup, since we do not use $\rho$ such that $\rho^*$ may be infinite
(discussed in Corollary \ref{cor:strong-convexity-smoothness} in Appendix \ref{app:general-lemmas}),
we have to care about only the condition when $\ell^*_{y_i}(-\alpha_i)$ becomes infinite.
In the examples of loss functions in Table \ref{tab:loss-functions} we present conditions
for $\ell^*_{y_i}(-\alpha_i) < +\infty$.

\begin{table*}[tp]
\caption{Examples of sample-sparse loss functions}
\label{tab:loss-functions}
\begin{tabular}{lcccccc}
\hline
\multicolumn{1}{c}{$\ell_y(t)$} & $y\in$ & $\phantom{\Bigl[}\check{X}\phantom{\Bigr]}$ & $\ell^*_y(t)$ & $\ell^*_{y_i}(-\alpha_i) < +\infty$ if & ${\cal Z}[\ell_y]$ \\
\hline
%
$\max\{0, 1-t\}$ & $\{-1, +1\}$ & $\bm y\sqtimes X$ & $t$ ($-1\leq t\leq 0$) & $0\leq\alpha_i\leq 1$ & $(1, +\infty)$ \\
\multicolumn{2}{l}{\quad(\emph{Hinge loss})} & & $+\infty$ (otherwise) & & \\
%
$(\max\{0, 1-t\})^2$ & $\{-1, +1\}$ & $\bm y\sqtimes X$ & $(t^2 + 4t)/4$ ($t\leq 0$) & $\alpha_i\geq 0$ & $[1, +\infty)$ \\
\multicolumn{2}{l}{\quad(\emph{Squared hinge loss})} & & $+\infty$ (otherwise) & & \\
%
$\max\{0, |t - y| - \varepsilon\}$ & $\mathbb{R}$ & $X$ & $t(y - \varepsilon)$ ($-1 \leq t\leq 0$) & $-1\leq\alpha_i\leq 1$ & $(-\varepsilon, \varepsilon)$ \\
\multicolumn{2}{l}{\quad(\emph{$\varepsilon$-insensitive loss};} & & $t(y + \varepsilon)$ ($0 \leq t\leq 1$) & & \\
\multicolumn{2}{l}{\quad $\varepsilon > 0$: hyperparameter)} & & $+\infty$ (otherwise) & & \\
$(\max\{0, |t - y| - \varepsilon\})^2$ & $\mathbb{R}$ & $X$ & $(t^2 + 4yt)/4 + |t|\varepsilon$ & $\alpha_i\in\mathbb{R}$ & $[-\varepsilon, \varepsilon]$ \\
\multicolumn{2}{l}{\quad(\emph{squared $\varepsilon$-insensitive loss};} & & & & \\
\multicolumn{2}{l}{\quad $\varepsilon > 0$: hyperparameter)} & & & & \\
\hline
\end{tabular}
\end{table*}

\section{PROOFS OF SECTION \ref{sec:safe-screening}}

\subsection{Proof of Lemma \ref{lem:gap-sphere-primal}} \label{app:gap-sphere-primal}

\begin{proof} \citep{ndiaye2015gap}
\begin{align*}
\| \hat{\bm\beta} - \bm\beta^{*(\bm w)} \|_2 
	& \leq \sqrt{\frac{2}{\lambda}[P_{\bm w}(\hat{\bm\beta}) - P_{\bm w}(\bm\beta^{*(\bm w)})]}
	\tag{$\because$ setting $f\gets P_{\bm w}$ in Lemma \ref{lem:strong-convexity-sphere}} \\
	& = \sqrt{\frac{2}{\lambda}[P_{\bm w}(\hat{\bm\beta}) - D_{\bm w}(\bm\alpha^{*(\bm w)})]}
	\tag{$\because$ \eqref{eq:strong-duality}} \\
	& \leq \sqrt{\frac{2}{\lambda}[P_{\bm w}(\hat{\bm\beta}) - D_{\bm w}(\hat{\bm\alpha})]}.
	\tag{$\because$ $\bm\alpha^{*(\bm w)}$ is a maximizer of $D_{\bm w}$}
\end{align*}
\end{proof}

\subsection{Proof of Lemma \ref{lem:gap-sample-screening}} \label{app:gap-sample-screening}

\begin{proof}
Due to \eqref{eq:KKT-primal2dual}, if $\partial\ell_{y_i}(\check{X}_{i:}\bm\beta^{*(\bm w)}) = \{0\}$ is assured, then $\alpha_i^{*(\bm w)} = 0$ is assured.
Since we do not know $\bm\beta^{*(\bm w)}$ but know ${\cal B}^{*(\bm w)}$ (Lemma \ref{lem:gap-sphere-primal}),
we can assure $\alpha_i^{*(\bm w)} = 0$ if $\bigcup_{\bm\beta\in{\cal B}^{*(\bm w)}} \partial\ell_{y_i}(\check{X}_{i:}\bm\beta) = \{0\}$ is assured.
Noticing that $\partial\ell_{y_i}$ is monotonically increasing\footnote{Since $\partial\ell_{y_i}$ is a multi-valued function, the monotonicity must be defined accordingly: we call a multi-valued function $F: \mathbb{R}\to 2^{\mathbb{R}}$ is monotonically increasing if, for any $t < t^\prime$, $F$ must satisfy ``$\forall s\in F(t)$, $\forall s^\prime\in F(t^\prime)$:~$s\leq s^\prime$''.}, we have
\begin{align*}
& \bigcup_{\bm\beta\in{\cal B}^{*(\bm w)}} \partial\ell_{y_i}(\check{X}_{i:}\bm\beta) = \{0\}
	\quad\Leftrightarrow\quad
	\bigcup_{\bm\beta\in{\cal B}^{*(\bm w)}} \check{X}_{i:}\bm\beta \subseteq {\cal Z}[\ell_{y_i}]
	\quad\Leftrightarrow\quad
	[\min_{\bm\beta\in{\cal B}^{*(\bm w)}} \check{X}_{i:}\bm\beta, \max_{\bm\beta\in{\cal B}^{*(\bm w)}} \check{X}_{i:}\bm\beta] \subseteq {\cal Z}[\ell_{y_i}] \\
& \Leftrightarrow\quad \left[\check{X}_{i:}\hat{\bm\beta} - \|\check{X}_{i:}\|_2 r(\bm w, \kappa, \hat{\bm\beta}, \hat{\bm\alpha}),~
	\check{X}_{i:}\hat{\bm\beta} + \|\check{X}_{i:}\|_2 r(\bm w, \kappa, \hat{\bm\beta}, \hat{\bm\alpha}) \right]
	\subseteq {\cal Z}[\ell_{y_i}].
	\tag{$\because$ Lemma \ref{lem:optimize-linear}}
\end{align*}
\end{proof}

\section{PROOFS AND ADDITIONAL DISCUSSIONS OF SECTION \ref{sec:DRSS-examples}}

\subsection{Use of Strongly-Convex Regularization Functions other than L2-Regularization} \label{app:regularization-functions}

Let us discuss using regularization functions other than L2-regularization,
on condition that it is $\kappa$-strongly convex to apply DRSSS (Lemma \ref{lem:gap-sphere-primal}).
In case of L2-regularization, as seen in \eqref{eq:duality-gap-l2reg}, the term
$\rho^*((\bm w\sqtimes\check{X})^\top \bm\alpha^{*(\tilde{\bm w})})$
in the duality gap \eqref{eq:duality-gap} becomes a quadratic function with respect to $\bm w$.
However, this may largely differ for other regularization functions, even if it is $\kappa$-strongly convex
(to apply DRSSS; Lemma \ref{lem:gap-sphere-primal}).

As a famous example, let us see the case of the \emph{elastic net} regularization \citep{zou05regularization}: with hyperparameters $\lambda > 0$ and $\lambda^\prime > 0$, the regularization function is defined and its convex conjugate is calculated as
\begin{align*}
& \rho(\bm\beta) := \frac{\lambda}{2}\|\bm\beta\|_2^2 + \lambda^\prime\|\bm\beta\|_1, \\
& \rho^*(\bm v) = \frac{1}{2\lambda}\sum_{j=1}^d\left[\max\{0, |v_j| - \lambda^\prime \}\right]^2.
\end{align*}
Here, we can see that the maximization of
$\rho^*((\bm w\sqtimes\check{X})^\top \bm\alpha^{*(\tilde{\bm w})}) = (1/2\lambda)\sum_{j=1}^d\left[\max\{0, |\bm w\otimes\bm\alpha^{*(\tilde{\bm w})}\otimes \check{X}_{:j} | - \lambda^\prime \}\right]^2$
with respect to $\bm w$ seems not to be easy.

Then we discuss a (sufficient) condition of regularization functions so that
weighted RERM can accept both kernelization and SSS.

\begin{lemma}
In weighted RERM of Definition \ref{def:WRERM},
if $\rho$ is described as follows, it can accept both SSS in Lemma \eqref{lem:gap-sample-screening}
and kernelization in Section \ref{sec:safe-screening-kernelized}:
$\rho$ is described by an increasing function ${\cal H}: \mathbb{R}_{\geq 0}\to\mathbb{R}$ and
a positive constant $\kappa$ as $\rho(\bm\beta) = (\kappa/2)\|\bm\beta\|_2^2 + {\cal H}(\|\bm\beta\|_2)$.
\end{lemma}

\begin{proof}
According to the \emph{generalized representer theorem} \citep{scholkopf2001generalized},
weighted RERM can be kernelized if $\rho$ is described by a strictly increasing function
${\cal G}: \mathbb{R}_{\geq 0}\to\mathbb{R}$ as $\rho(\bm\beta) = {\cal G}(\|\bm\beta\|_2)$.

Combining the condition to apply SSS (Lemma \ref{lem:gap-sphere-primal}) that
$\rho$ should be $\kappa$-strongly convex,
we have the conclusion.
\end{proof}

However, we have to notive that, in order to extend it to DRSSS, the same problem as the example of
the elastic net will appear.

\subsection{Maximizing Linear and Convex Quadratic Functions in Hyperball Constraint} \label{app:maximize-convex-quadratic}

First we present Theorem \ref{thm:maximize-convex-quadratic-sketch} in more specific descriptions.

\begin{theorem} \label{thm:maximize-convex-quadratic}
The maximization problem in Theorem \ref{thm:maximize-convex-quadratic-sketch}:
\begin{align}
& \max_{\bm w\in{\cal W}} \bm w^\top A \bm w + 2\bm b^\top\bm w, \tag{\eqref{eq:maximize-convex-quadratic} restated} \\
& \text{where}
	\quad {\cal W} := \{ \bm w\in\mathbb{R}^n \mid \|\bm w - \tilde{\bm w}\|_2\leq S \}, \nonumber\\
& \phantom{\text{where}}
	\quad \tilde{\bm w}\in\mathbb{R}^n,
	\quad \bm b\in\mathbb{R}^n,
	\quad S > 0, \nonumber\\
& \phantom{\text{where}}
	\quad A\in\mathbb{R}^{n\times n}:~\text{symmetric, positive semidefinite, nonzero} \nonumber
\end{align}
can be achieved by the following procedure.
First, we define 
$Q\in\mathbb{R}^{n\times n}$ and $\Phi := \mathrm{diag}(\phi_1, \phi_2, \dots, \phi_n)$
as the eigendecomposition of $A$ such that $A = Q^\top\Phi Q$,
$Q$ is orthogonal ($QQ^\top = Q^\top Q = I$).
Also, let $\bm\xi := -\Phi Q\tilde{\bm w} - Q\bm b \in\mathbb{R}^n$, and 
\begin{align}
& {\cal T}(\nu) = \sum_{i=1}^n \left(\frac{\xi_i}{\nu - \phi_i}\right)^2. \label{eq:lagrangian-result-nu}
\end{align}
Then, the maximization \eqref{eq:maximize-convex-quadratic} is equal to the largest value among them:
\begin{itemize}
\item For each $\nu$ such that
	${\cal T}(\nu) = S^2$ (see Lemma \ref{lem:find-invsq}),
	the value $\nu S^2 + (\nu\tilde{\bm w} + \bm b)^\top Q^\top(\Phi - \nu I)^{-1}\bm\xi + \bm b^\top\tilde{\bm w}$, and
\item For each $\nu\in\{\phi_1, \phi_2, \dots, \phi_n\}$ (duplication removed)
	such that ``$\forall i\in[n]:~\phi_i = \nu \Rightarrow \xi_i = 0$'',
	the value
	\begin{align*}
	& \max_{\bm\tau\in\mathbb{R}^n} [\nu S^2 + (\nu\tilde{\bm w} + \bm b)^\top Q^\top \bm\tau + \bm b^\top\tilde{\bm w}],\\
	& \text{subject to}
		\quad\forall i\in{\cal F}_\nu:\quad \tau_i = \frac{\xi_i}{\phi_i - \nu}, \nonumber\\
	& \phantom{\text{subject to}}
		\quad \sum_{i\in{\cal U}_\nu} \tau_i^2 = S^2 - \sum_{i\in{\cal F}_\nu} \tau_i^2, \nonumber \\
	& \text{where}
		\quad {\cal U}_\nu := \{ i \mid i\in[n],~\phi_i = \nu \},
		\quad {\cal F}_\nu := [n]\setminus{\cal U}_\nu.
	\end{align*}
	(Note that the maximization is easily computed by Lemma \ref{lem:optimize-linear}.)
\end{itemize}
The computation can be done in $O(n^3)$ time.
\end{theorem}

This can be proved as follows.

\begin{lemma} \label{lem:maximize-convex-quadratic-stationary}
For the optimization problem \eqref{eq:maximize-convex-quadratic},
its stationary points are obtained as the solution of the following equations with respect to $\bm w$ and $\nu\in\mathbb{R}$:
\begin{align}
& A \bm w + \bm b - \nu(\bm w - \tilde{\bm w}) = \bm 0, \label{eq:lagrangian-target} \\
& \|\bm w - \tilde{\bm w}\|_2 = S. \label{eq:lagrangian-sphere}
\end{align}
Also, when both \eqref{eq:lagrangian-target} and \eqref{eq:lagrangian-sphere} are satisfied,
the function to be maximized is calculated as
\begin{align}
\bm w^\top A \bm w + 2\bm b^\top\bm w
	= \nu S^2 + (\nu \tilde{\bm w} + \bm b)^\top (\bm w - \tilde{\bm w}) + \tilde{\bm w}^\top \bm b.
	\label{eq:maximization-replaced}
\end{align}
\end{lemma}

\begin{proof}
First, $\bm w^\top A \bm w + 2\bm b^\top\bm w$ is convex and not constant.
Then we can show that \eqref{eq:maximize-convex-quadratic} is optimized in $\{ \bm w\in\mathbb{R}^n \mid \|\bm w - \tilde{\bm w}\|_2 = S \}$, that is, at the surface of the hyperball ${\cal W}$
(Theorem 32.1 of \citep{rockafellar1970convex}). This proves \eqref{eq:lagrangian-sphere}.
Moreover, with the fact, we write the Lagrangian function with Lagrange multiplier $\nu\in\mathbb{R}$ as:
\begin{align*}
L(\bm w, \nu) := \bm w^\top A \bm w + 2\bm b^\top\bm w - \nu(\|\bm w - \tilde{\bm w}\|_2^2 - S^2).
\end{align*}
Then, due to the property of Lagrange multiplier,
the stationary points of \eqref{eq:maximize-convex-quadratic} are obtained as
\begin{align*}
& \frac{\partial L}{\partial\bm w} = 2 A \bm w + 2\bm b - 2 \nu(\bm w - \tilde{\bm w}) = 0, \\
& \frac{\partial L}{\partial\nu} = \|\bm w - \tilde{\bm w}\|_2^2 - S^2 = 0,
\end{align*}
where the former derives \eqref{eq:lagrangian-target}.

Finally we show \eqref{eq:maximization-replaced}.
If both \eqref{eq:lagrangian-target} and \eqref{eq:lagrangian-sphere} are satisfied,
\begin{align}
\bm w^\top A\bm w + 2 \bm b^\top \bm w
& = \bm w^\top(\nu(\bm w - \tilde{\bm w}) - \bm b) + 2 \bm b^\top \bm w \tag{$\because$~\eqref{eq:lagrangian-target}}\\
& = \nu\bm w^\top(\bm w - \tilde{\bm w}) + \bm b^\top \bm w \nonumber\\
& = \nu(\bm w - \tilde{\bm w})^\top(\bm w - \tilde{\bm w}) + \nu\tilde{\bm w}^\top(\bm w - \tilde{\bm w}) + \bm b^\top(\bm w - \tilde{\bm w}) + \bm b^\top\tilde{\bm w} \nonumber\\
& = \nu S^2 + \nu\tilde{\bm w}^\top(\bm w - \tilde{\bm w}) + \bm b^\top(\bm w - \tilde{\bm w}) + \bm b^\top\tilde{\bm w} \tag{$\because$~\eqref{eq:lagrangian-sphere}}\\
& = \nu S^2 + (\nu\tilde{\bm w} + \bm b)^\top(\bm w - \tilde{\bm w}) + \bm b^\top\tilde{\bm w} \tag{\eqref{eq:maximization-replaced} restated}
\end{align}
\end{proof}

\begin{lemma} \label{lem:find-invsq}
Under the same definitions as Theorem \ref{thm:maximize-convex-quadratic},
The equation ${\cal T}(\nu) = S^2$ can be solved by the following procedure:
Let $\bm e := [e_1, e_2, \dots, e_N]$ ($N\leq n$, $k\neq k^\prime \Rightarrow e_k\neq e_{k^\prime}$) be a sequence of indices such that
\begin{enumerate}
\item $e_k\in[n]$ for any $k\in[N]$,
\item $i\in[n]$ is included in $\bm e$ if and only if $\xi_i\neq 0$, and
\item $\phi_{e_1} \leq \phi_{e_2} \leq \dots \leq \phi_{e_N}$.
\end{enumerate}
Note that, if $\phi_{e_k} < \phi_{e_{k+1}}$ ($k\in[N-1]$), then ${\cal T}(\nu)$ is a convex function
in the interval $(\phi_{e_k}, \phi_{e_{k+1}})$ with $\lim_{\nu\to\phi_{e_k}+0} = \lim_{\nu\to\phi_{e_{k+1}}-0} = +\infty$.
Then, unless $N = 0$, each of the following intervals contains just one solution of ${\cal T}(\nu) = S^2$:
\begin{itemize}
\item Intervals $(-\infty, \phi_{e_1})$ and $(\phi_{e_N}, +\infty)$.
\item Let $\nu^{\#(k)} := \targmin_{\phi_{e_k}<\nu<\phi_{e_{k+1}}} {\cal T}(\nu)$.
	For each $k\in[N-1]$ such that $\phi_{e_k}<\phi_{e_{k+1}}$,
	\begin{itemize}
	\item intervals $(\phi_{e_k}, \nu^{\#(k)})$ and $(\nu^{\#(k)}, \phi_{e_{k+1}})$ if ${\cal T}(\nu^{\#(k)}) < S^2$,
	\item interval $[\nu^{\#(k)}, \nu^{\#(k)}]$ (i.e., point) if ${\cal T}(\nu^{\#(k)}) = S^2$.
	\end{itemize}
\end{itemize}
It follows that ${\cal T}(\nu) = S^2$ has at most $2n$ solutions.
\end{lemma}
By Lemma \ref{lem:find-invsq}, in order to compute the solution of ${\cal T}(\nu) = S^2$,
we have only to compute $\nu^{\#(k)}$ by Newton method or the like,
and to compute the solution for each interval by Newton method or the like.
We show an example of ${\cal T}(\nu)$ in Figure \ref{fig:sum_equals_R2}.

\begin{proof}[Proof of Lemma \ref{lem:find-invsq}]
We show the statements in the lemma that, if $\phi_{e_k} < \phi_{e_{k+1}}$ ($k\in[N-1]$), then ${\cal T}(\nu)$ is a convex function
in the interval $(\phi_{e_k}, \phi_{e_{k+1}})$ with $\lim_{\nu\to\phi_{e_k}+0} = \lim_{\nu\to\phi_{e_{k+1}}-0} = +\infty$. Then the conclusion immediately follows.

The latter statement clearly holds. The former statement is proved by directly computing the derivative.
\begin{align*}
& \frac{d}{d\nu}{\cal T}(\nu)
	= \frac{d}{d\nu}\sum_{i=1}^n \left(\frac{\xi_i}{\nu - \phi_i}\right)^2
	= -2\sum_{i=1}^n \frac{\xi_i^2}{(\nu - \phi_i)^3}.
\end{align*}
It is an increasing function with respect to $\nu$, as long as $\nu$ does not match any of $\{\phi_i\}_{i=1}^n$
such that $\xi_i\neq 0$. So it is convex in the interval $\phi_{e_k} < \nu < \phi_{e_{k+1}}$.
\end{proof}

\begin{figure}[t]
\includegraphics[width=\hsize]{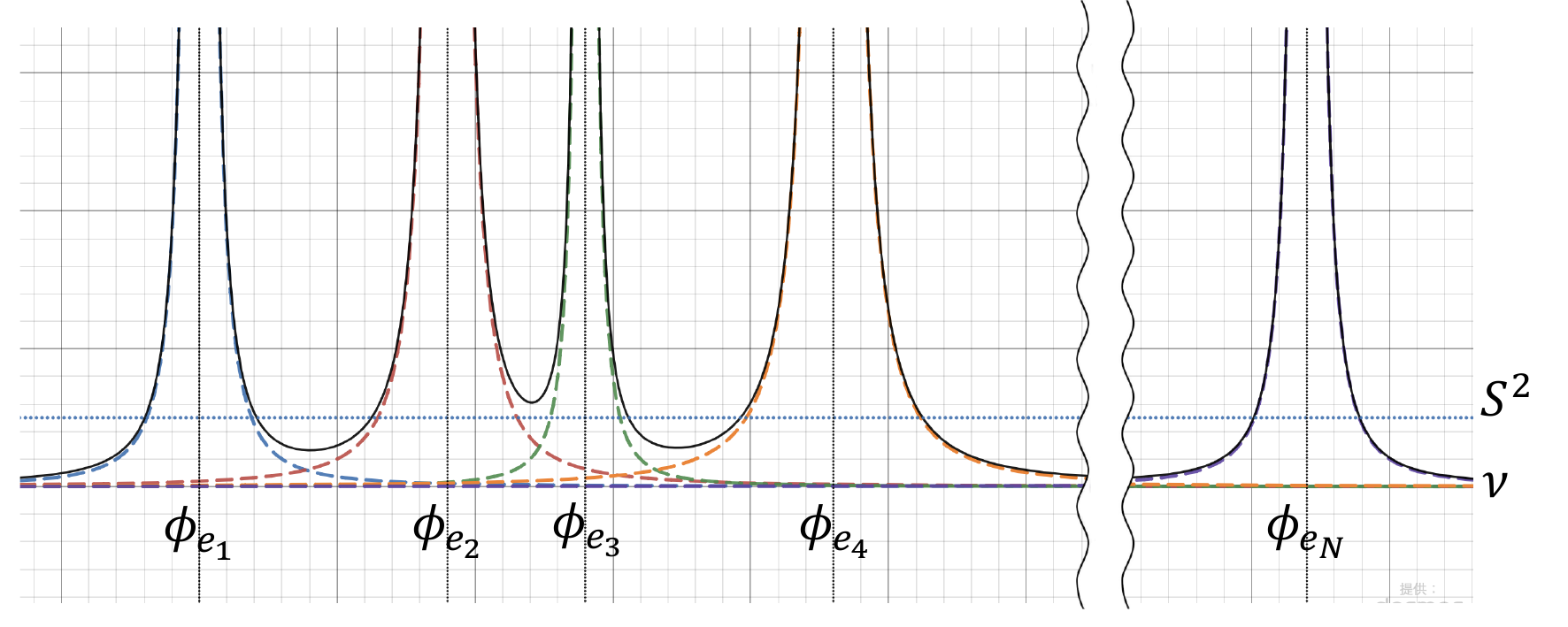}
\vspace{-2em}
\caption{An example of the expression ${\cal T}(\nu)$ (black solid line) in Lemmas \ref{thm:maximize-convex-quadratic} and \ref{lem:find-invsq}.
Colored dash lines denote terms in the summation $(\xi_{e_k}/(\nu - \phi_{e_k}))^2$.
We can see that, given an interval $(\phi_{e_k}, \phi_{e_{k+1}})$ ($k\in[N-1]$), the function is convex.}
\label{fig:sum_equals_R2}
\end{figure}

Now we are prepared to prove Theorem \ref{thm:maximize-convex-quadratic}.

\begin{proof}[Proof of Theorem \ref{thm:maximize-convex-quadratic}]
The condition \eqref{eq:lagrangian-target} is calculated as
\begin{align*}
& A \bm w + \bm b = \nu(\bm w - \tilde{\bm w}), \\
& (A - \nu I)(\bm w - \tilde{\bm w}) = -A\tilde{\bm w} - \bm b.
\end{align*}
Here, let us apply eigendecomposition of $A$, denoted by $A = Q^\top\Phi Q$,
where $Q\in\mathbb{R}^{n\times n}$ is orthogonal ($QQ^\top = Q^\top Q = I$) and
$\Phi := \mathrm{diag}(\phi_1, \phi_2, \dots, \phi_n)$ is a diagonal matrix consisting of eigenvalues of $A$.
Such a decomposition is assured to exist since $A$ is assumed to be symmetric and positive semidefinite.
Then,
\begin{align}
& (Q^\top\Phi Q - \nu I)(\bm w - \tilde{\bm w}) = -Q^\top\Phi Q\tilde{\bm w} - \bm b, \nonumber\\
& Q^\top(\Phi - \nu I) Q (\bm w - \tilde{\bm w}) = -Q^\top\Phi Q\tilde{\bm w} - \bm b, \nonumber\\
& (\Phi - \nu I) \bm\tau = \bm\xi,
	\quad
	(\text{where}
	\quad \bm\tau := Q (\bm w - \tilde{\bm w}),
	\quad \bm\xi := -\Phi Q\tilde{\bm w} - Q\bm b \in\mathbb{R}^n,) \nonumber \\
& \forall i\in[n]:\quad (\phi_i - \nu) \tau_i = \xi_i. \label{eq:equation-by-tau}
\end{align}
Note that we have to be also aware of the constraint
\begin{align}
S = \|\bm\tau\|_2 = \sqrt{\bm\tau^\top\bm\tau} = \sqrt{(\bm w - \tilde{\bm w})^\top Q^\top Q (\bm w - \tilde{\bm w})} = \|\bm w - \tilde{\bm w}\|_2. \label{eq:tau-constraint}
\end{align}

Here, we consider these two cases.
\begin{enumerate}
\item First, consider the case when $(\Phi - \nu I)$ is nonsingular, that is, when $\nu$ is different from any of $\phi_1, \phi_2, \dots, \phi_n$. Then, from \eqref{eq:tau-constraint} we have
	\begin{align}
	& S^2
		= \|\bm\tau\|_2 = \sum_{i=1}^n \tau_i^2
		= \sum_{i=1}^n \left(\frac{\xi_i}{\nu - \phi_i}\right)^2 \quad\bigl(=: {\cal T}(\nu)\bigr). \label{eq:lagrangian-result-nu-equation}
	\end{align}
	So, values of \eqref{eq:maximize-convex-quadratic} for all stationary points with respect to $\bm w$ and $\nu$ (on condition that $(\Phi - \nu I)$ is nonsingular) can be obtained by computing \eqref{eq:maximization-replaced} for each $\nu$ satisfying \eqref{eq:lagrangian-result-nu-equation}, that is,
	\begin{itemize}
	\item for such $\nu$ computing $\bm\tau$ by \eqref{eq:equation-by-tau}, and
	\item computing \eqref{eq:maximization-replaced} as
		$\nu S^2 + (\nu\tilde{\bm w} + \bm b)^\top(\bm w - \tilde{\bm w}) + \bm b^\top\tilde{\bm w}
		= \nu S^2 + (\nu\tilde{\bm w} + \bm b)^\top Q^\top \bm\tau + \bm b^\top\tilde{\bm w}$.
	\end{itemize}

\item Secondly, consider the case when $(\Phi - \nu I)$ is nonsingular, that is, when $\nu$ is equal to one of $\phi_1, \phi_2, \dots, \phi_n$.
	First, given $\nu$, let ${\cal U}_\nu := \{ i \mid i\in[n],~\phi_i = \nu \}$ be the indices of $\{\phi_i\}_i$ equal to $\nu$ (this may include more than one indices), and ${\cal F}_\nu := [n]\setminus{\cal U}_\nu$. Note that, by assumption, ${\cal U}_\nu$ is not empty.
	Then, all stationary points of \eqref{eq:maximize-convex-quadratic} with respect to $\bm w$ and $\nu$ (on condition that $(\Phi - \nu I)$ is singular) can be found by computing the followings for each $\nu\in\{\phi_1, \phi_2, \dots, \phi_n\}$ (duplication excluded):
	\begin{itemize}
	\item If $\xi_i\neq 0$ for at least one $i\in{\cal U}_\nu$, the equation \eqref{eq:equation-by-tau} cannot hold.
	\item If $\xi_i = 0$ for all $i\in{\cal N}_\nu$, the equation \eqref{eq:equation-by-tau} may hold.
		So we calculate $\bm\tau$ that maximizes \eqref{eq:maximize-convex-quadratic} as follows:
		\begin{itemize}
		\item Fix $\tau_i = \xi_i / (\phi_i - \nu)$ for $i\in{\cal F}_\nu$.
		\item Set the constraint $\sum_{i\in{\cal U}_\nu} \tau_i^2 = S^2 - \sum_{i\in{\cal F}_\nu} \tau_i^2$ (due to \eqref{eq:tau-constraint}).
		\item Maximize \eqref{eq:maximize-convex-quadratic} with respect to $\{\tau_i\}_{i\in{\cal U}_\nu}$
			under the constraints above. Here, by \eqref{eq:maximization-replaced} we have only to calculate
			\begin{align}
			& \max_{\bm\tau\in\mathbb{R}^n} [\nu S^2 + (\nu\tilde{\bm w} + \bm b)^\top(\bm w - \tilde{\bm w}) + \bm b^\top\tilde{\bm w}], \label{eq:singular-max}\\
			& \text{subject to}
				\quad\forall i\in{\cal F}_\nu:\quad \tau_i = \frac{\xi_i}{\phi_i - \nu}, \nonumber\\
			& \phantom{\text{subject to}}
				\quad \sum_{i\in{\cal U}_\nu} \tau_i^2 = S^2 - \sum_{i\in{\cal F}_\nu} \tau_i^2, \nonumber
			\end{align}
			which is easily computed by Lemma \ref{lem:optimize-linear}.
			The value of the maximization result is equal to that of \eqref{eq:maximize-convex-quadratic} on condition that $\nu$ is specified above.
		\end{itemize}
		So, collecting these result and taking the largest one, the maximization (on condition that $(\Phi - \nu I)$ is singular) is completed.
	\end{itemize}
\end{enumerate}
Taking the maximum of the two cases, we have the maximization result of \eqref{eq:maximize-convex-quadratic}.

Finally we show the computational cost.
All matrix multiplications in the computation process can be done in at most $O(n^2)$ time,
and we additionally need the following computations:
the eigendecomposition of $A$ and
the solution of ${\cal T}(\nu) = S^2$ using Newton method or the like (Lemma \ref{lem:find-invsq}).
The former can be done in $O(n^3)$ time.
The latter requires Newton method computations for $O(n)$ time,
and each of Newton method computation requires evaluations of ${\cal T}(\nu)$ that requires $O(n^2)$ time.
So, if we assume that the number of repetition of Newton method is constant,
it can be done in $O(n^2)$ time.
This concludes that the computational cost is $O(n^3)$.
\end{proof}

\section{DETAILS OF EXPERIMENTS}

\subsection{Experimental Environments and Implementation Information} \label{app:implementation}

We used the following computers for experiments:
For experiments except for the image dataset,
we run experiments on a computer with Intel Xeon Silver 4214R (2.40GHz) CPU and 64GB RAM.
For experiments using the image dataset,
we run experiments on a computer with Intel(R) Xeon(R) Gold 6338 (2.00GHz) CPU, NVIDIA RTX A6000 GPU and 1TB RAM.

Procedures are implemented in Python, mainly with the following libraries:
\begin{itemize}
\item {\em NumPy} \citep{harris2020array}: Matrix and vector operations
\item {\em CVXPY} \citep{diamond2016cvxpy}: Convex optimizations (training computations with weights)
\item {\em SciPy} \citep{2020SciPy-NMeth}: Solving equations to maximize the quadratic convex function in Theorem \ref{thm:maximize-convex-quadratic} (by module {\em optimize.root\_scalar})
\item {\em PyTorch} \citep{paszke2017automatic}: Defining the source neural network (which will be converted to a kernel by NTK) for image prediction
\item {\em neural-tangents} \citep{neuraltangents2020}: NTK
\end{itemize}

\subsection{Data and Learning Setup} \label{app:experimental-setup}

The criteria of selecting datasets (Table \ref{tab:dataset-SS}) and detailed setups are as follows:
\begin{itemize}
\item All of the datasets are downloaded from LIBSVM dataset \citep{libsvmDataset}.
	We used training datasets only if test datasets are provided separately
	(``splice'' and ``svmguide1'').
\item We selected datasets from LIBSVM dataset containing 100 to 10,000 samples,
	100 or fewer features, and the area under the curve (AUC) of
	the receiver operating characteristic (ROC) is 0.9 or higher
	for the regularization strengths ($\lambda$) we examined
	so that they tend to facilitate more effective sample screening.
\item In the table, the column ``$d$'' denotes the number of features including the intercept feature (Remark \ref{rem:intercept}).
\end{itemize}

The choice of the regularization hyperparameter $\lambda$, based on the characteristics of the data, is as follows:
We set $\lambda$ as $n$, $n\times 10^{-0.5}$, $n\times 10^{-1.0}$, $\ldots$, $n\times 10^{-3.0}$.
	(For DRSSS with DL, we set 1000 instead of $n$.)
	This is because the effect of $\lambda$ gets weaker for larger $n$.

The choice of the hyperparameter in RBF kernel is fixed as follows: we set $\zeta = d^\prime * \mathbb{V}(Z)$ as suggested in {\tt sklearn.svm.SVC} of \emph{scikit-learn} \citep{scikit-learn}, where $\mathbb{V}$ denotes the elementwise sample variance.


\subsection{Setup of Randomly Generated Weights} \label{app:random-weights}

Given $S$ of $\|\bm w - \tilde{\bm w}\|_2 \leq S$, in order to randomly generate $\bm w$ satisfying the above, we took the following procedure.

We generate $\bm w$ by $\bm w\gets\tilde{\bm w} + (S/\|\bm v\|_2)\bm v$ ($\bm v\in\mathbb{R}^n$),
where $\bm v$ is drawn from an $n$-dimensional standard normal distribution.
Here, the method above generates only $\bm w$ such that $\|\bm w - \tilde{\bm w}\|_2 = S$,
we intentionally did this since we would like to examine the behavior of methods when $\bm w$ is
as away from $\tilde{\bm w}$ as possible under the constraint.

\subsection{All Experimental Results of Section \ref{sec:experiment-weight-changes}} \label{app:experiment}

We show all experimental results of safe sample screening rates for linear and RBF-kernel SVMs in Figure \ref{fig:result-all}.
For the dataset ``phishing'', we could not calculate for RBF kernel due to memory constraint in our environment. 
%

\begin{figure}[tp]
\begin{tabular}{cc}
\begin{minipage}[b]{0.47\hsize}\centering Dataset: australian, Linear kernel\\\includegraphics[width=0.7\hsize]{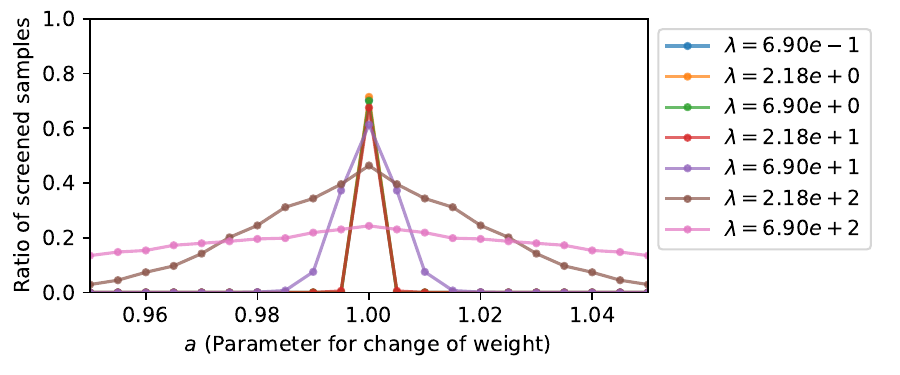}\end{minipage}
&
\begin{minipage}[b]{0.47\hsize}\centering Dataset: australian, RBF kernel\\\includegraphics[width=0.7\hsize]{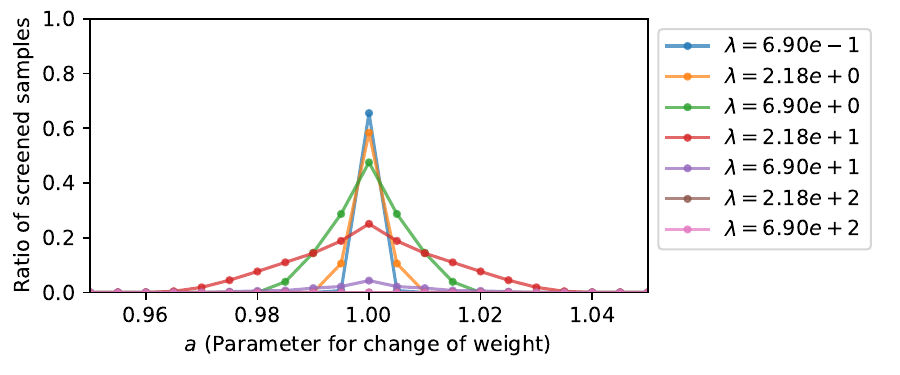}\end{minipage}
\\
\begin{minipage}[b]{0.47\hsize}\centering Dataset: breast-cancer, Linear kernel\\\includegraphics[width=0.7\hsize]{sample-screening-redrawn/linear-svm_ss_screening_rate_Label_WeightL2Range_breast-cancer.pdf}\end{minipage}
&
\begin{minipage}[b]{0.47\hsize}\centering Dataset: breast-cancer, RBF kernel\\\includegraphics[width=0.7\hsize]{sample-screening-redrawn/rbf-svm_ss_screening_rate_Label_WeightL2Range_breast-cancer.pdf}\end{minipage}
\\
\begin{minipage}[b]{0.47\hsize}\centering Dataset: heart, Linear kernel\\\includegraphics[width=0.7\hsize]{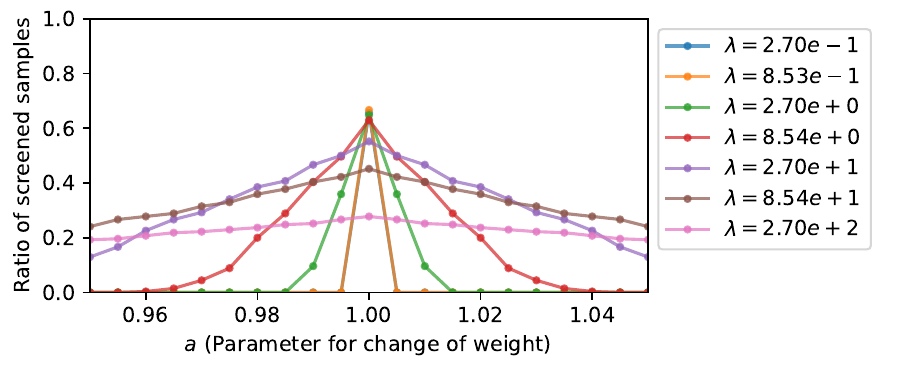}\end{minipage}
&
\begin{minipage}[b]{0.47\hsize}\centering Dataset: heart, RBF kernel\\\includegraphics[width=0.7\hsize]{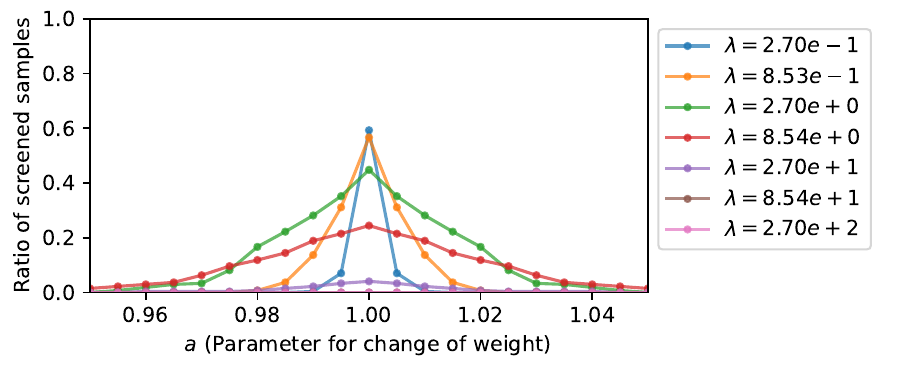}\end{minipage}
\\
\begin{minipage}[b]{0.47\hsize}\centering Dataset: ionosphere, Linear kernel\\\includegraphics[width=0.7\hsize]{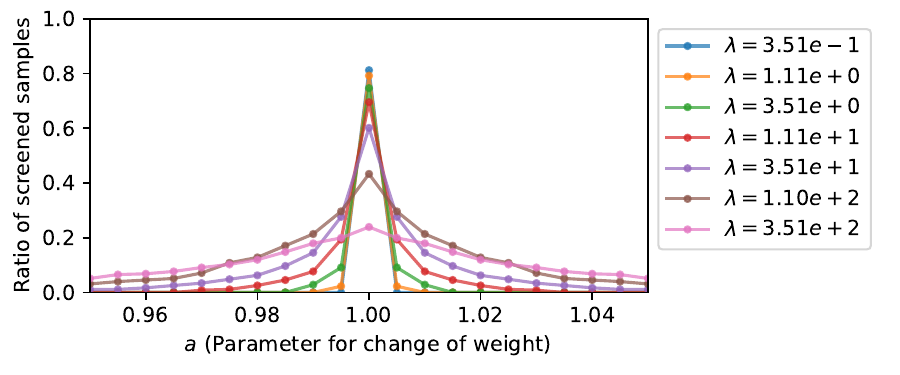}\end{minipage}
&
\begin{minipage}[b]{0.47\hsize}\centering Dataset: ionosphere, RBF kernel\\\includegraphics[width=0.7\hsize]{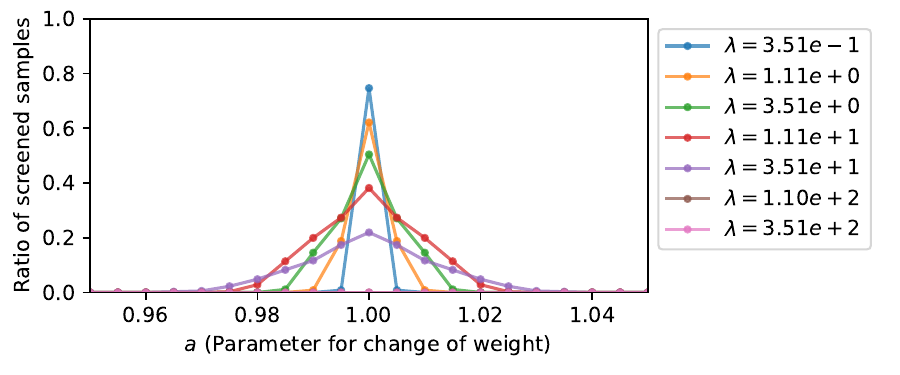}\end{minipage}
\\
\begin{minipage}[b]{0.47\hsize}\centering Dataset: phishing, Linear kernel\\\includegraphics[width=0.7\hsize]{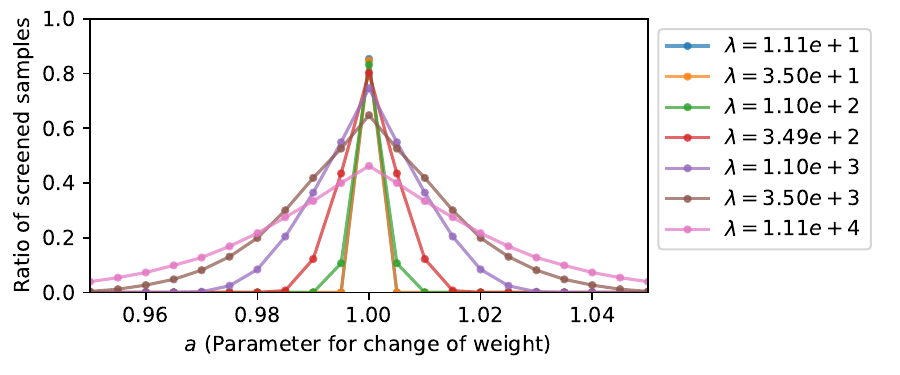}\end{minipage}
&
\begin{minipage}[b]{0.47\hsize}~
\end{minipage}
\\
\begin{minipage}[b]{0.47\hsize}\centering Dataset: sonar, Linear kernel\\\includegraphics[width=0.7\hsize]{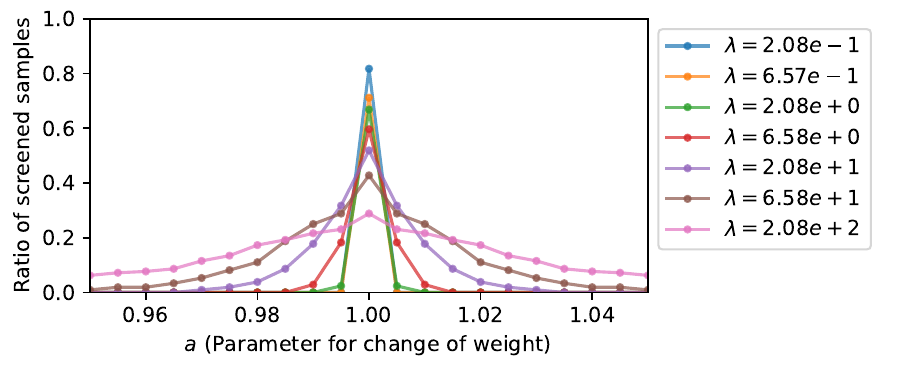}\end{minipage}
&
\begin{minipage}[b]{0.47\hsize}\centering Dataset: sonar, RBF kernel\\\includegraphics[width=0.7\hsize]{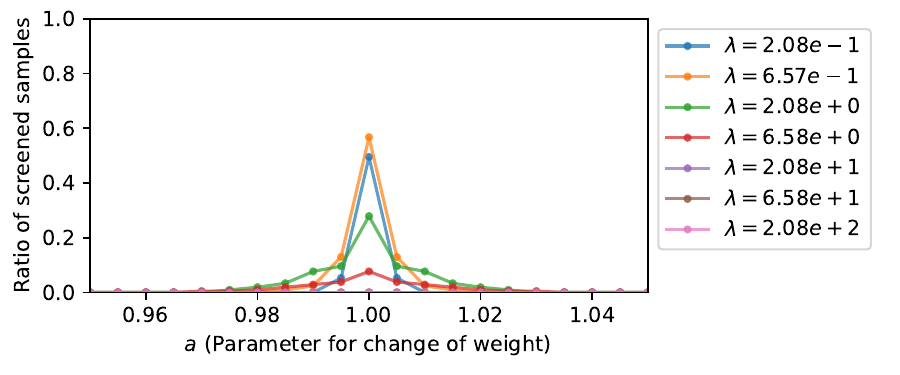}\end{minipage}
\\
\begin{minipage}[b]{0.47\hsize}\centering Dataset: splice, Linear kernel\\\includegraphics[width=0.7\hsize]{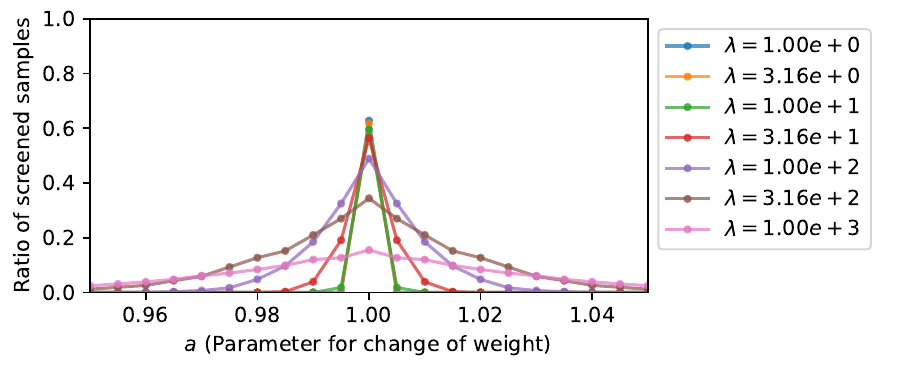}\end{minipage}
&
\begin{minipage}[b]{0.47\hsize}\centering Dataset: splice, RBF kernel\\\includegraphics[width=0.7\hsize]{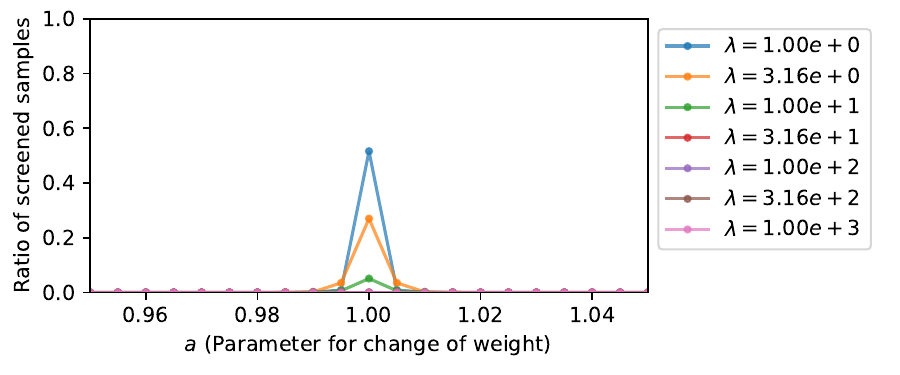}\end{minipage}
\\
\begin{minipage}[b]{0.47\hsize}\centering Dataset: svmguide1, Linear kernel\\\includegraphics[width=0.7\hsize]{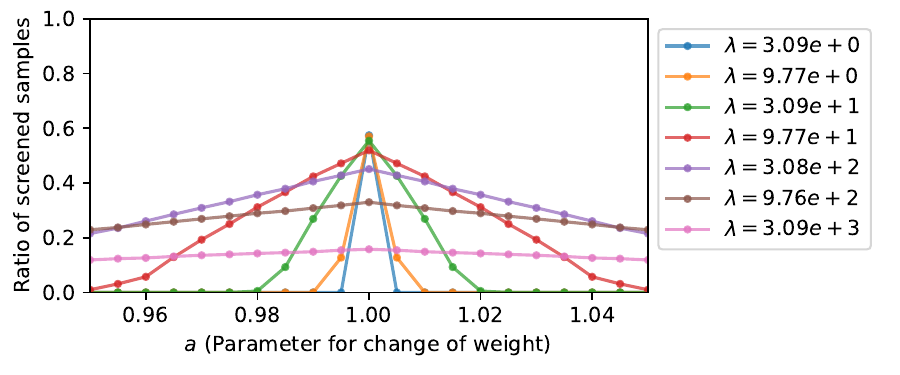}\end{minipage}
&
\begin{minipage}[b]{0.47\hsize}\centering Dataset: svmguide1, RBF kernel\\\includegraphics[width=0.7\hsize]{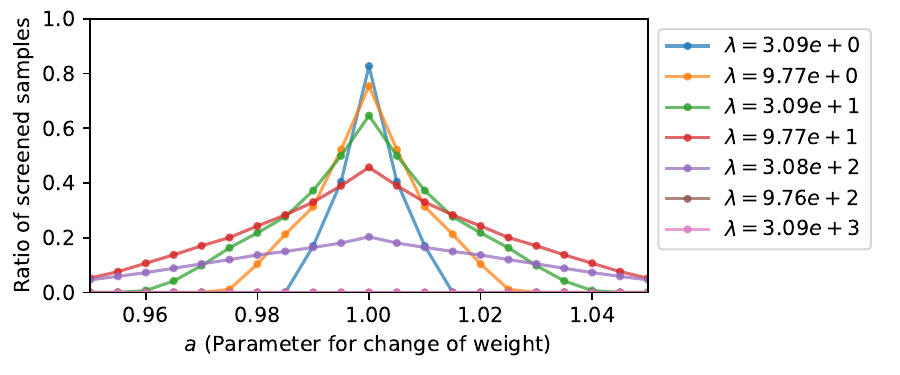}\end{minipage}
\end{tabular}
\caption{Safe sample screening rates for linear- and RBF-kernel SVMs, under the settings described in Section \ref{sec:experiment} and Appendix \ref{app:experimental-setup}.}
\label{fig:result-all}
\end{figure}

\end{document}